\documentclass[11pt, twocolumn]{article}
\usepackage[left=0.87in, right=0.87in, top=1in, bottom=1in]{geometry}
\usepackage[T1]{fontenc}
\usepackage[utf8]{inputenc}

\usepackage{microtype}
\usepackage{graphicx}
\usepackage{subcaption}
\usepackage{booktabs}
\usepackage{xspace}
\usepackage{hyperref}
\usepackage{amsfonts, amsmath, amsthm, amssymb}
\usepackage{mathtools}
\usepackage{bm}
\usepackage{xargs}
\usepackage{lipsum}
\usepackage[capitalize]{cleveref}
\usepackage{thm-restate}
\usepackage{color}
\usepackage{algorithm, algorithmic}
\usepackage{natbib}

\title{\textbf{Pure Exploration with Structured Preference Feedback}}
\author{\textbf{Shubham Gupta$^1$, Aadirupa Saha$^1$, and Sumeet Katariya$^2$} \\
    $^1$Indian Institute of Science, Bangalore \\
    $^2$Amazon, Palo Alto, CA \\
    \texttt{[shubhamg, aadirupa]@iisc.ac.in}, \texttt{katsumee@amazon.com}
}
\date{}

\renewcommand{\cite}[1]{\citep{#1}}

\newcommand{\mul}[2]{\langle #1, #2 \rangle}

\newcommand{\ba}{\mathbf{a}}
\newcommand{\bb}{\mathbf{b}}
\newcommand{\be}{\mathbf{e}}
\newcommand{\bg}{\mathbf{g}}
\newcommand{\bM}{\mathbf{M}}
\newcommand{\bu}{\mathbf{u}}
\newcommand{\bx}{\mathbf{x}}
\newcommand{\bX}{\mathbf{X}}
\newcommand{\by}{\mathbf{y}}
\newcommand{\bmmu}{\bm{\mu}}
\newcommand{\bmepsilon}{\bm{\epsilon}}
\newcommand{\bmtheta}{\bm{\theta}}

\newcommand{\rmE}{\mathrm{E}}
\newcommand{\rmKL}[2]{\mathrm{KL}(#1 \, \vert\vert \, #2)}
\newcommand{\rmP}{\mathrm{P}}
\newcommand{\bbR}{\mathbb{R}}
\newcommand{\bbI}[1]{\mathbb{I}\{#1\}}

\newcommand{\calA}{\mathcal{A}}
\newcommand{\calC}{\mathcal{C}}
\newcommand{\calG}{\mathcal{G}}
\newcommand{\calN}{\mathcal{N}}

\newcommand{\htheta}{\hat{\bmtheta}}
\newcommand{\hDelta}{\hat{\Delta}}

\newcommand{\supt}{^{(t)}}
\newcommand{\sups}{^{(s)}}
\newcommand{\supS}{^S}

\newcommand{\sums}{\sum_{s=1}^t}
\newcommand{\sumS}{\sum_{S \subseteq_K \calA}}

\newcommand{\abs}[1]{\vert #1 \vert}

\newcommandx{\norm}[3][2={}, 3={}]{\vert\vert #1 \vert\vert_{#2}^{#3}}
\newcommandx{\normtwo}[1]{\norm{#1}[2]}
\newcommandx{\normtwosq}[1]{\norm{#1}[2][2]}
\newcommandx{\normvtinv}[1]{\norm{#1}[{V\supt}^{-1}]}

\newtheorem{lemma}{Lemma}
\newtheorem{theorem}{Theorem}
\newtheorem{assumption}{Assumption}
\newtheorem{corollary}{Corollary}
\newtheorem*{remark}{Remark}

\DeclareMathOperator*{\argmax}{arg\,max}
\DeclareMathOperator*{\argmin}{arg\,min}

\newcommand{\lambdamax}[1]{\lambda_{\mathrm{max}}(#1)}
\newcommand{\lambdamin}[1]{\lambda_{\mathrm{min}}(#1)}
\mathchardef\mhyphen="2D 
\newcommand{\algname}{{\tt BAI$\mhyphen$Lin$\mhyphen$MNL}\xspace}
\newcommand{\algnameadaptive}{{\tt BAI$\mhyphen$Lin$\mhyphen$MNL$\mhyphen$Adap}\xspace}

\begin{document}
\maketitle

\begin{abstract}
    We consider the problem of pure exploration with subset-wise preference feedback, which contains $N$ arms with features. The learner is allowed to query subsets of size $K$ and receives feedback in the form of a noisy winner. The goal of the learner is to identify the best arm efficiently using as few queries as possible. This setting is relevant in various online decision-making scenarios involving human feedback such as online retailing, streaming services, news feed, and online advertising; since it is easier and more reliable for people to choose a preferred item from a subset than to assign a likability score to an item in isolation. To the best of our knowledge, this is the first work that considers the subset-wise preference feedback model in a structured setting, which allows for potentially infinite set of arms. We present two algorithms that guarantee the detection of the best-arm in $\tilde{O} (\frac{d^2}{K \Delta^2})$ samples with probability at least $1 - \delta$, where $d$ is the dimension of the arm-features and $\Delta$ is the appropriate notion of utility gap among the arms. We also derive an instance-dependent lower bound of $\Omega(\frac{d}{\Delta^2} \log \frac{1}{\delta})$ which matches our upper bound on a worst-case instance. Finally, we run extensive experiments to corroborate our theoretical findings, and observe that our adaptive algorithm stops and requires up to 12x fewer samples than a non-adaptive algorithm.
\end{abstract}


\section{Introduction}
\label{section:introduction}

In the classical multi-armed bandits (MAB) setting, the agent pulls an arm at each time step and receives the corresponding reward \cite{AuerEtAl:2002:FiniteTimeAnalysisOfTheMultiarmedBanditProblem}. However, it is often more convenient for humans to choose a preferred item from a set than to assign a real-valued likability score to an item in isolation. The dueling bandit problem studies a variant of the MAB framework where the agent selects two arms at each step and obtains noisy feedback indicating the winner of a comparison between the two choices \cite{YueEtAl:2012:TheKArmedDuelingBanditsProblem}. This paper considers a more general feedback model, also known as the Multinomial Logit  Model (MNL model) \cite{Marden:1996:AnalyzingAndModelingRankData,SahaEtAl:2019:PACBattlingBanditsInThePlackettLuceModel}, where the agent selects $K \geq 2$ arms at each step and observes a noisy winner as feedback.

We also consider a structured setting where each arm is associated with a $d$-dimensional feature vector \cite{Auer:2002:UsingConfidenceBoundsForExploitationExplorationTradeoffs,LiEtAl:2010:AContextualBanditApproachToPersonalizedNewsArticleRecommendation}. The MNL feedback model with structured arms is a natural choice for several applications like online retailing, streaming services, news feed, and online advertising, which contain a large repository of arms. For example, in online advertising, users click on an ad out of a subset of K ads displayed to them. The features of the ad can be the image and text embedding of its contents learned by an off-the-shelf neural network. The structured feedback setting is well-suited for such applications where the number of arms $N$ is potentially infinite and new arms are constantly added.

We assume that the reward for an arm with feature vector $\bx$ is $\bx^\intercal \bmtheta^\ast$, where $\bmtheta^*$ is an unknown parameter. We study the pure-exploration or best-arm-identification problem of finding the best arm with high confidence \cite{BubeckEtAl:2009:PureExplorationInMultiArmedBanditsProblems,SoareEtAl:2014:BestArmIdentificationInLinearBandits}. This is different from the more commonly studied regret minimization problem. In pure exploration, the goal is to choose the subsets adaptively at each time so as to identify the best arm using as few queries as possible. 

We explain next the challenges in designing a provably-optimal algorithm for this problem, and our contributions to overcome them. As opposed to the standard linear bandits setting, the feedback under the MNL model is a non-linear function of the arm feature vectors (see Section \ref{section:problem_setting}). Moreover, this feedback is vector-valued and the elements of this vector are not independent. This makes it difficult to construct a confidence interval for the unknown parameter $\bmtheta^*$ using existing strategies \cite{LiEtAl:2017:ProvablyOptimalAlgorithmsForGeneralizedLinearContextualBandits,KazerouniEtAl:2019:BestArmIdentificationInGeneralizedLinearBandits}, since they are designed for scalar link functions. Using the mean-value theorem for vector-valued functions, we derive new concentration bounds for terms involving the feedback vectors, where existing strategies fail due to the dependence between the elements (see Section \ref{section:confidence_bound} for a more detailed description). The derived bound (Theorem \ref{theorem:confidence_interval}) can be of independent interest.

We use the confidence interval to design our algorithm \algname{}, which is a static allocation strategy, which means that the sequence of arm pulls is not influenced by the observed rewards \citep{SoareEtAl:2014:BestArmIdentificationInLinearBandits}. In the MNL setting, the number of actions available to the learner grows exponentially with $K$ as every subset of $K$ arms is an action. \algname{} offers an efficient two-layer greedy solution for selecting the subsets across time steps, and the arms within each subset. We prove that this greedy strategy is probably correct (returns the correct arm with failure probability at most $\delta$). We also derive an $\tilde{O}(\frac{d^2}{K \Delta_{\min}^2})$ upper bound on the sample complexity of this greedy strategy. Here, $\Delta_{\min} = \min_{\ba \neq \ba^*} \mul{\bmtheta^*}{\ba^* - \ba}$, where $\ba^*$ is the best arm. We also develop and analyze an adaptive allocation strategy \algnameadaptive{}, that is adaptive in batches. It stops and requires up to \textbf{12x} fewer samples than \algname{} in our experiments!

We then show that our \emph{algorithm and upper bound is minimax optimal by deriving an instance-dependent lower bound} for the sample complexity of any algorithm on a subclass of problems where the arms are linearly independent. We do so using the change-of-measure argument \citep{KaufmannEtAl:2016:OnTheComplexityOfBestArmIdentificationInMultiArmedBanditModels}, which requires the construction of an adversarial problem instance that has a different best-arm and that deviates from the given problem instance specified by $\bmtheta^*$ only on a handful of actions. This is non-trivial under the MNL feedback model where each action corresponds to a subset of K arms, and thus changing an arm changes multiple actions. A second layer of complexity exists in the structured setting since changing $\bmtheta^\ast$ changes the reward of multiple arms. We construct an adversarial problem instance and use it to derive an instance-dependent $\Omega(\frac{d}{\Delta_{\min}^2})$ lower bound in Theorem \ref{theorem:lower_bound}, and show that our proposed algorithm is minimax optimal.

Finally, we conduct experiments to a) verify that the sample complexity indeed scales with parameters $d,K$ as predicted by our upper bound, b) study the robustness of our algorithms to deviations from the MNL feedback model, and c) compare our algorithms to other baselines when $K=2$ (there are no known algorithms for our setting when $K>2$). We observe that our adaptive algorithm stops and requires up to 12x fewer samples than a non-adaptive algorithm.

\textbf{Related work. } 
The best-arm identification problem has been extensively studied in the classical MAB setting \cite{EvenDarEtAl:2006:ActionEliminationAndStoppingConditionsForTheMABandRLProblems,AudibertEtAl:2010:BestArmIdentificationInMultiArmedBandits,KalyanakrishnanEtAl:2012:PACSubsetSelectionInStochasticMultiarmedBandits,BubeckEtAl:2013:MultipleIdentificationsInMultiArmedBandits,JamiesonEtAl:2014:LilUCBAnOptimalExplorationAlgorithmForMAB}. However, as opposed to the case of independent arms, \citet{SoareEtAl:2014:BestArmIdentificationInLinearBandits} note that even pulling known sub-optimal arms may help in identifying the best arm in linear bandits setting, thus requiring a different strategy. Following the seminal work of \citet{SoareEtAl:2014:BestArmIdentificationInLinearBandits}, several algorithms for determining the best arm in linear bandits have surfaced \cite{XuEtAl:2018:AFullyAdaptiveAlgorithmForPureExplorationInLinearBandits,TaoEtAl:2018:BestArmIdentificationInLinearBanditsWithLinearDimensionDependency,FiezEtAl:2019:SequentialExperimentalDesignForTransductiveLinearBandits,ZakiEtAl:2019:TowardsOptimalAndEfficientBestArmIdentificationInLinearBandits,DegenneEtAl:2020:GamificationOfPureExplorationForLinearBandits,JedraEtAl:2020:OptimalBestArmIdentificationInLinearBandits,KatzSamuelsEtAl:2020:AnEmpiricalProcessApproachToTheUnionBound,ZakiEtAl:2020:ExplicitBestArmIdentificationInLinearBanditsUsingNoRegretLearners}. All of them assume the standard reward model for linear bandits where the agent observes the reward for the pulled arm. Instead, we use the MNL feedback model. 

MNL model has been studied from two perspectives in the literature. In the first case, each subset of $K$ arms has an average \emph{revenue} (average reward of arms in the subset weighted by their probability of being chosen under MNL model) associated with it, and the goal of best-arm identification is to choose the subset that maximizes this revenue \cite{RusmevichientongEtAl:2010:DynamicAssortmentOptimizationWithAMultinomialLogitChoiceModelAndCapacityConstraint}. See the dynamic assortment selection literature for examples \cite{AgrawalEtAl:2017:ThompsonSamplingForTheMNLBandit, AgrawalEtAl:2019:MNLBandit, ChenEtAl:2020:DynamicAssortmentOptimizationWithChangingContextualInformation}. In the second case, the goal of best-arm identification is to identify a single best arm (arm with the highest reward) as opposed to identifying a subset with the highest revenue \cite{Luce:1959:IndividualChoiceBehavior, Plackett:1975:TheAnalysisOfPermutations, SzorenyiEtAl:2015:OnlineRankElicitationForPlackettLuce, ChenEtAl:2018:ANearlyInstanceOptimalAlgorithmForTopkRankingUnderTheMultinomialLogitModel, RenEtAl:2018:ANearlyInstanceOptimalAlgorithmForTopkRankingUnderTheMultinomialLogitModel, SahaEtAl:2019:PACBattlingBanditsInThePlackettLuceModel}. This paper belongs to the second category. In this setting, best-arm identification has been studied by \citet{SahaEtAl:2019:PACBattlingBanditsInThePlackettLuceModel} in the standard MAB setting where arms do not have features. Several authors have studied regret minimization under the MNL feedback model \cite{AgrawalEtAl:2017:ThompsonSamplingForTheMNLBandit,AgrawalEtAl:2019:MNLBandit,OhIyengar:2019:ThompsonSamplingForMultinomialLogitContextualBandits,ChenEtAl:2020:DynamicAssortmentOptimizationWithChangingContextualInformation}, but we focus on best-arm identification.

Best-arm identification has also been studied for combinatorial bandits under the standard bandit feedback \cite{KurokiEtAl:2020:PolynomialTimeAlgorithmsForMultipleArmIdentificationWithFullBanditFeedback,RejwanMansour:2020:TopKCombinatorialBanditsWithFullBanditFeedback, DuEtAl:2021:CombinatorialPureExplorationWithFullBanditOrPartialLinearFeedback} and partial linear feedback \cite{DuEtAl:2021:CombinatorialPureExplorationWithFullBanditOrPartialLinearFeedback}, whereas we use MNL feedback. \citet{ChenEtAl:2020:CombinatorialPureExplorationForDuelingBandits} consider the dueling bandit model, but assume that arms are independent. The batched bandit setting \cite{JunEtAl:2016:TopArmIdentificationInMultiArmedBanditsWithBatchArmPulls} also requires the user to select a subset of arms to pull. However, unlike in MNL bandits, the learner observes reward for each arm in this setting. Finally, \citet{KazerouniEtAl:2019:BestArmIdentificationInGeneralizedLinearBandits} perform best-arm identification in generalized linear bandits. Their algorithm can be applied to our setting only when $K = 2$. We consider the more general case with $K \geq 2$. To the best of our knowledge, this is the first paper to study best-arm identification in a linear bandits setting under the MNL feedback model.


\section{Problem Setting}
\label{section:problem_setting}
Let $[n]$ denote the set $\{1, 2, \dots, n\}$ for any integer $n$, and $\mul{\mathbf{p}}{\mathbf{q}}$ denote the standard inner product between the vectors $\mathbf{p}$ and $\mathbf{q}$.

Let $\calA = \{\ba_1, \ba_2, \dots, \ba_N\}$ be a set of $N$ arms, each specified by a $d$-dimensional feature vector $\ba_i \in \bbR^d$. At each step, the agent selects an action which corresponds to a subset of $K$ arms and observes the winner of a competition among the chosen arms. We use $\bx_1\sups, \bx_2\sups, \dots, \bx_K\sups \in \calA$ to denote the arm vectors selected by the agent at time $s$ and $\by\sups \in \{0, 1\}^K$ to denote a one-hot encoded vector specifying the competition winner. Note that the index $i$ denotes the global arm index in $\ba_i$ but the local subset index in $\bx_i\sups$ and $y_i\sups$.

An instance of the linear-MNL-bandit problem is a tuple $(\calA, \bmtheta^\ast, K, \delta)$ where $\calA$ is the set of $N$ arms, $K$ is the size of the subset, and $\delta>0$ is the probability of failure. $\calA, K$ and $\delta$ are known to the agent. The parameter $\bmtheta^\ast \in \bbR^d$ is unknown to the agent, and we assume that the environment samples $\by\sups$ such that for all times $s$,
\begin{equation}
    \label{eq:feedback_model}
    \rmP_{\bmtheta^*}(y_i\sups = 1 \vert \bX\sups) = \mu_i\sups(\bmtheta^*).
\end{equation}
Here $\bX\sups \in \bbR^{d \times K}$ is a matrix that has $\bx_1\sups, \bx_2\sups, \dots, \bx_K\sups$ as it columns, and $\mu_i\sups(\bmtheta)$ is defined as 
\begin{equation}
    \label{eq:mu_def}
    \mu_i\sups(\bmtheta) = \frac{\exp \mul{\bmtheta}{\bx_i\sups}}{\sum_{j = 1}^K \exp \mul{\bmtheta}{\bx_j\sups}}, \text{ for } i = 1, 2, \dots, K.
\end{equation}
This feedback model is also known as the Plackett-Luce (PL) model or the Multinomial Logit (MNL) model \cite{Luce:1959:IndividualChoiceBehavior,Plackett:1975:TheAnalysisOfPermutations}. Let $\ba^* = \arg \max_{\ba_i \in \calA}\mul{\bmtheta^*}{\ba_i}$ be the unique best arm. A solution to the linear-MNL-bandit is an algorithm which given a probability of failure $\delta>0$, chooses actions $\bX\sups$ for all times $s=1,2,\dots,\tau$ up to a stopping time $\tau$, and upon stopping returns an arm $\hat{\ba}$ such that 
$$\rmP(\hat{\ba} = \ba^* \land \tau < \infty) \geq 1 - \delta.$$
This setting is known as the fixed confidence setting \citep{garivier2016optimal}, and the goal is to identify the best-arm using as few samples as possible. Without loss of generality, we index the arms in the order of their mean rewards such that $\mul{\bmtheta^*}{\ba_1} \geq \mul{\bmtheta^*}{\ba_2} \geq \dots \geq \mul{\bmtheta^*}{\ba_N}$.


\section{Confidence Bound}
\label{section:confidence_bound}
In this section, we derive a confidence-bound for the unknown parameter $\bmtheta^\ast$ based on a series of $t$ observations $\{(\bX\sups, \by\sups)\}_{s=1}^t$. This confidence bound is used in the design and analysis of our algorithms \algname{} and \algnameadaptive{}. This bound is a novel contribution that can be of independent interest.

Let $\htheta\supt \in \bbR^d$ be the maximum likelihood estimate of parameter $\bmtheta^*$ obtained using data $\{(\bX\sups, \by\sups)\}_{s=1}^t$ collected till time $t$. Assuming that $\by\sups$'s follow the distribution given in \eqref{eq:feedback_model}, one can show that (see \cref{appendix:mle}) $\htheta\supt$ satisfies  
$$\sums \bX\sups (\by\sups - \bmmu\sups(\htheta\supt)) = 0.$$
To find the required confidence set, we show a high-probability bound on $\abs{\mul{\htheta\supt - \bmtheta^*}{\bx}}$ for an arbitrary direction $\bx \in \bbR^d$. 

Our derivation uses a strategy similar to that of \citet{LiEtAl:2017:ProvablyOptimalAlgorithmsForGeneralizedLinearContextualBandits} who show a similar bound for generalized linear models with a \emph{scalar} link function. Instead, we have the softmax function specified in \eqref{eq:mu_def} as our link function. Viewed as a function of $\bmtheta$, it maps a $d$-dimensional vector to a $K$-dimensional simplex element. This vector-to-vector mapping creates difficulties in applying the standard mean-value theorem, which forms a vital component of the proof in \citet{LiEtAl:2017:ProvablyOptimalAlgorithmsForGeneralizedLinearContextualBandits}.  We use the mean-value theorem for vector-valued functions. Another challenge stems from the fact that the elements of the (one-hot encoded) feedback vectors are not independent. This requires a new strategy to derive concentration bounds for the terms that involve these vectors. We define appropriate assumptions for a vector-valued link function (Assumptions \ref{assumption:kappa_def} and \ref{assumption:lambda_tilde_def}), and derive new concentration bounds for terms involving feedback vectors (Lemmas \ref{lemma:high_prob_alpha_bound} and \ref{lemma:term1_bound} in Appendix \ref{appendix:confidence_bound}).

As is typical of similar results \cite{LiEtAl:2017:ProvablyOptimalAlgorithmsForGeneralizedLinearContextualBandits,VanDerVaartEtAl:2000:AsymptoticStatisticsVolume3}, we require regularity assumptions on the link function's first and second-order derivatives. Unlike a scalar link function, these derivatives are represented by a matrix and a tensor, respectively, in our case. We need the following definitions to specify our assumptions. Let $F: \bbR^d \times \bbR^d \rightarrow \bbR^{d \times d}$ be defined as:
\begin{equation}
    \label{eq:F_def}
    F(\bmtheta_1, \bmtheta_2) = \sums \bX\sups \bM\sups(\bmtheta_1, \bmtheta_2) \bX\sups{'},
\end{equation}
where,
\begin{align}
    \label{eq:M_def}
    \bM\sups(\bmtheta_1, \bmtheta_2) = \int_0^1 \Big[&\mathrm{diag}(\bmmu\sups(\bmtheta)) - \notag \\
    &\bmmu\sups(\bmtheta)\bmmu\sups(\bmtheta)' \Big]_{\bmtheta = q\bmtheta_1 + (1 - q)\bmtheta_2} \mathrm{dq}.
\end{align}

\begin{assumption}
    \label{assumption:kappa_def}
    Let $B_\alpha = \{\bmtheta \in \bbR^d : \normtwo{\bmtheta - \bmtheta^*} \leq \alpha\}$ for a given $\alpha > 0$. We assume that there exists a $\kappa_\alpha > 0$ such that:
    $$\kappa_\alpha = \sup \{\kappa \in \bbR: \forall \bmtheta \in B_\alpha, F(\bmtheta, \bmtheta^*) \succeq \kappa V\supt\},$$ where $V\supt = \sums \bX\sups \bX\sups{'}$.
\end{assumption}

\begin{assumption}
    \label{assumption:lambda_tilde_def}
    Define $f\sups: \bbR^d \rightarrow \bbR^{K \times K}$ as $f\sups(\bmtheta) = \bM\sups (\bmtheta, \bmtheta^*)$ and let $S\sups = \int_0^1 [\nabla_{\bmtheta} f\sups]_{\bmtheta = q\htheta\supt + (1 - q)\bmtheta^*} \mathrm{dq}$ be a $K \times K \times d$ dimensional tensor. We use $S\sups(i)$ to denote the $i^{th}$ slice of dimension $K \times K$. We assume that there exists a $\tilde{\lambda} > 0$ such that $$\max_{s \in [t], \,i \in [d]} \lambdamax{S\sups(i)} \leq \tilde{\lambda}.$$
\end{assumption}

The quantity $\bM\sups$ defined in \eqref{eq:M_def} depends on the first-order derivative of the softmax function. Assumption \ref{assumption:kappa_def} ensures that the first derivative is strictly positive in a neighborhood of $\bmtheta^*$. Similarly, $S\sups$ depends on the second-order derivative of the softmax function and Assumption \ref{assumption:lambda_tilde_def} is analogous to having an upper bound on the second-order derivative in case of a scalar link function. We state our confidence bound for $\bmtheta^\ast$ next, and prove it in \cref{appendix:confidence_bound}. 
\begin{restatable}{theorem}{confidencebound}
    \label{theorem:confidence_interval}
    Assume that $\normtwosq{\ba_i} \leq 1$ for all $i \in [N]$ and Assumptions \ref{assumption:kappa_def} and \ref{assumption:lambda_tilde_def} hold. For a fixed sequence $\{\bX\sups\}_{s \leq t}$ define $V\supt$ as in Assumption \ref{assumption:kappa_def}, and further assume that
    $$\lambdamin{V\supt} \geq 64 \frac{\tilde{\lambda}^2 d}{\kappa_\alpha^4} (d + \log(1/\delta)),$$
    then, with probability at least $1 - 3\delta$, for any $\bx \in \bbR^d$, $$\abs{\mul{\bx}{\htheta\supt - \bmtheta^*}} \leq \frac{8}{\kappa_\alpha}  \sqrt{d + \log(1/\delta)} \,\, \normvtinv{\bx}.$$ 
\end{restatable}

The assumption on $\lambdamin{V\supt}$ holds for large enough $t$ and is a necessary assumption for consistency of estimating linear and generalized linear models \cite{LaiEtAl:1982:LeastSquaresEstimatesInStochasticRegressionModels,FahrmeirEtAl:1985:ConsistencyAndAsymptoticNormalityOfTheMaximumLikelihoodEstimatorInGeneralizedLinearModels,BickelEtAl:2009:SimultaneousAnalysisOfLassoAndDantzigSelector}. Let $\calG$ denote the set of pairwise differences between arms vectors in $\calA$, i.e., $\calG = \{\bx - \by : \bx, \by \in \calA\}$. The following corollary, obtained using a union bound over all $t > 0$ and gaps $\bg \in \calG$, is a simple consequence of Theorem \ref{theorem:confidence_interval}.

\begin{corollary}
    \label{corollary:confidence_bound}
    Assume $\exists\, t{'} > 0$ such that for all $t \geq t{'}$, $$\lambdamin{V\supt} \geq 64 \frac{\tilde{\lambda}^2 d}{\kappa_\alpha^4} (d + \log(3N^2t^2/\delta)),$$
    then, for a fixed sequence $\{\bX\sups\}_{s > 0}$,
    \begin{align*}
        \rmP\Big(\forall\, t &\geq t{'}, \forall\, \bg \in \calG, \abs{\mul{\bg}{\htheta\supt - \bmtheta^*}} \leq \\ 
        &\frac{8}{\kappa_\alpha}  \sqrt{d + \log(3N^2t^2/\delta)} \,\, \normvtinv{\bg}\Big) \geq 1 - \delta.
    \end{align*}
\end{corollary}

Equipped with the confidence bound, we now present our algorithms \algname and \algnameadaptive next.


\section{Algorithm}
\label{section:algorithm}

In this section, we propose a static allocation strategy \algname{} (where the chosen action does not depend on the observed rewards) and an adaptive allocation strategy \algnameadaptive, which are the linear-MNL counterparts of similar strategies in \citep{SoareEtAl:2014:BestArmIdentificationInLinearBandits}. Both strategies use Theorem \ref{theorem:confidence_interval} to construct the confidence sets. We first discuss derivations of the stopping criterion and action selection strategy, and then present the pseudo-code of the two algorithms. 


\subsection{Stopping Criterion}
\label{section:stopping_criteria}

For each arm $\ba_i \in \calA$, define $\calC_i = \{\bmtheta \in \bbR^d : \forall j \in [N], \mul{\bmtheta}{\ba_i} \geq \mul{\bmtheta}{\ba_j}\}$ to be the set of parameters $\bmtheta$ for which $\ba_i$ is the optimal arm. At every step $t$, we use the observations collected till time $t$ to construct a confidence set $\calC\supt \subseteq \bbR^d$ such that $\rmP(\bmtheta^* \in \calC\supt) \geq 1 - \delta$. The following condition then provides a stopping criterion.
\begin{equation*}
    \exists\, \ba_i \in \calA \text{ such that } \calC\supt \subseteq \calC_i.
\end{equation*}
The criterion above is equivalent to the condition that $\exists \, \ba_i \in \calA$ such that $\forall \bmtheta \in \calC\supt$ and $\forall j \in [N]$, $\mul{\bmtheta}{\ba_i - \ba_j} \geq 0$. This, in turn, happens if and only if $\exists\, \ba_i \in \calA$ such that $\forall \bmtheta \in \calC\supt$ and $\forall j \in [N]$,
\begin{equation}
    \label{eq:stopping_cond_intermediate}
    \mul{\htheta\supt - \bmtheta}{\ba_i - \ba_j} \leq \mul{\htheta\supt}{\ba_i - \ba_j} \coloneqq \hDelta\supt_{ij}.
\end{equation}

Define the confidence set $\calC\supt = \{\bmtheta \in \bbR^d : \mul{\htheta\supt - \bmtheta}{\ba_i - \ba_j} \leq \frac{8}{\kappa_\alpha}  \sqrt{d + \log(3N^2t^2/\delta)} \,\, \normvtinv{\ba_i - \ba_j}, \forall i, j \in [N]\}$. The following condition implies the criteria in \eqref{eq:stopping_cond_intermediate} by the definition of $\calC\supt$: $\exists \, i \in [N]$ such that $\forall j \in [N]$,
\begin{equation}
    \label{eq:stopping_condition}
    \frac{8}{\kappa_\alpha}  \sqrt{d + \log(3N^2t^2/\delta)} \,\, \normvtinv{\ba_i - \ba_j} \leq \hDelta\supt_{ij}.
\end{equation}
By Corollary \ref{corollary:confidence_bound}, for all $t \geq t{'}$, $\bmtheta^* \in \calC\supt$ with probability $\geq  1 - \delta$. Thus, $\calC\supt$ contains the true parameter $\bmtheta^*$ when the stopping condition is encountered. Because $\calC\supt \subseteq \calC_i$ upon termination, the algorithm returns the correct arm $\ba^*$ with probability at least $1 - \delta$.

Next, we develop an action-selection strategy to find actions $\bX\sups$ that accelerate the process of eq. \eqref{eq:stopping_condition} being satisfied.


\subsection{Action-Selection Strategy}
\label{section:arm_selection_strategy}

The algorithm must select $K$ arms at each step to get a noisy feedback based on the MNL model. Since the goal is to satisfy \eqref{eq:stopping_condition} as fast as possible, an intuitive solution is to select $\bX\sups$ for $s \leq t$ such that
\begin{equation}
    \label{eq:ideal_arm_selection}
    \{\bX\sups\}_{s \leq t} = \argmin_{\{\bX\sups\}_{s \leq t}} \max_{i, j \in [N]} \frac{\normvtinv{\ba_i - \ba_j}}{\hDelta_{ij}\supt}.
\end{equation}
Unfortunately, we cannot do this because $\hDelta\supt_{ij}$ is based on the maximum-likelihood estimate $\htheta\supt$, which is calculated using the observed feedback $\by\sups$ for $s \leq t$. The sequence $\bX^{(1)}, \bX^{(2)}, \dots$ selected using eq. \eqref{eq:ideal_arm_selection} is adaptive which violates the requirement in Corollary \ref{corollary:confidence_bound} that the sequence $\{\bX\sups\}_{s > 0}$ be fixed.

Following \cite{SoareEtAl:2014:BestArmIdentificationInLinearBandits}, we instead solve the following relaxed optimization problem, which results in a static allocation strategy.
\begin{equation}
    \label{eq:arm_selection}
    \{\bX\sups\}_{s \leq t} = \argmin_{\{\bX\sups\}_{s \leq t}} \max_{i, j \in [N]} \normvtinv{\ba_i - \ba_j}.
\end{equation}
The strategy in \eqref{eq:ideal_arm_selection} attempts to select actions that shrink the confidence set $\calC\supt$ along the directions $\ba_i - \ba_j$ where the gaps $\hDelta\supt_{ij}$ are small. However, the action-selection strategy in \eqref{eq:arm_selection} aims at shrinking the confidence set uniformly across all directions in $\calG$.

\begin{algorithm}[t]
    \caption{\algname{}}
    \label{alg:static_XY_allocation}
    \begin{algorithmic}[1]
        \STATE \textbf{Input: }Set of arms $\calA$, subset size $K$, confidence $\delta > 0$, tuning parameter $t{'}$ 
        \STATE \textbf{Initialize: }$t \leftarrow 1$, $V^{(0)} \leftarrow \mathbf{0}_{d \times d}$, and $\calG \leftarrow \{\bx - \by : \bx, \by \in \calA\}$
        \WHILE[Random exploration]{$t < t{'}$}
            \STATE Set $\bx_k\supt = \ba_i$ for $\ba_i \stackrel{\mathrm{unif}}{\sim} \calA$ for all $k \in [K]$
            \STATE $V\supt \leftarrow V^{(t - 1)} + \bX\supt \bX\supt{'}.$
            \STATE $t \leftarrow t + 1$
        \ENDWHILE

        \WHILE{\eqref{eq:stopping_condition} is not true}
            \FOR[Greedy solution to \eqref{eq:arm_selection}]{$k \in [K]$}
                \STATE Set $\bx_k\supt = \argmin\limits_{\ba \in \calA} \max\limits_{\bg \in \calG} \norm{\bg}[(V^{(t-1)} + \ba \ba{'})^{-1}][2]$
                \STATE $V^{(t - 1)} \leftarrow V^{(t - 1)} + \bx_k\supt \bx_k\supt{'}$
            \ENDFOR
            \STATE $V\supt \leftarrow V^{(t - 1)}$,  $t \leftarrow t + 1$
            \STATE Estimate $\htheta\supt$ from data
        \ENDWHILE
        \STATE \textbf{Return:} $\argmax_{\ba \in \calA} \mul{\htheta\supt}{\ba}$
    \end{algorithmic}
\end{algorithm}

\subsection{\algname and \algnameadaptive}
\label{section:a_practical_algorithm}

The optimization problem in \eqref{eq:arm_selection} is combinatorial in nature as it requires one to choose actions from a given finite set $\calA^K$. Algorithm \ref{alg:static_XY_allocation} presents a greedy solution. After an initial $t{'}$ rounds of uniform exploration to satisfy the assumption in Corollary \ref{corollary:confidence_bound} (lines 3--7 in Algorithm \ref{alg:static_XY_allocation}), the algorithm sequentially chooses actions by solving a one-step greedy variant of the optimization problem in \eqref{eq:arm_selection}. Here, we assume that actions till step $t-1$ are fixed, and the goal is to select $\bX\supt$ to solve \eqref{eq:arm_selection}. The columns of $\bX\supt$ are also chosen one at a time in a greedy manner (lines 9--12). The output is observed after a subset of $K$ arms has been selected. The data is then used to estimate $\htheta\supt$ which is needed to compute the stopping criteria in \eqref{eq:stopping_condition}. 

The optimal solution of \eqref{eq:arm_selection} corresponds to the well-known $G$-optimal design from the experimental design literature \cite{SoareEtAl:2014:BestArmIdentificationInLinearBandits,Pukelsheim:2006:OptimalDesignOfExperiments}. The goal is to choose $Kt$ arms from a finite set $\calA$ to solve \eqref{eq:arm_selection}. While this discrete optimization problem in NP-hard, several approximate solutions exist \cite{BouhtouEtAl:2010:SubmodularityAndRandomizedRoundingTechniquesForOptimalExperimentalDesign} that yield objective values that are within a $(1 + \beta)$ multiplicative factor of the optimal objective value for some $\beta > 0$. 

In contrast to a static allocation strategy, an adaptive strategy is more desirable in practice as it can select actions that shrink the confidence set along ``important'' directions \cite{XuEtAl:2018:AFullyAdaptiveAlgorithmForPureExplorationInLinearBandits} rather than shrinking it uniformly as in \eqref{eq:arm_selection}. Unfortunately, as noted before, an adaptively chosen sequence violates the assumptions in Corollary \ref{corollary:confidence_bound}. We borrow an idea from \cite{SoareEtAl:2014:BestArmIdentificationInLinearBandits} and propose \algnameadaptive in \cref{alg:adaptive_XY_allocation_summary}, which runs in batches.
Each batch uses a static allocation strategy and the observed data is used to eliminate arms from consideration at the end of a batch, which makes the overall process adaptive. We only present a high-level pseudo-code in \cref{alg:adaptive_XY_allocation_summary} and refer the reader to \cref{appendix:adaptive_strategy} for the detailed code. We observe that \algnameadaptive requires up to 12x fewer samples than \algname in our experiments.

Note that while our algorithms are similar to \citet{SoareEtAl:2014:BestArmIdentificationInLinearBandits}, a key difference is that they only require the identity of the winner at each step, intead of the actual rewards for all arms. They are based on the new stopping criteria derived from Theorem \ref{theorem:confidence_interval}, and their analysis does not trivially follow from \citet{SoareEtAl:2014:BestArmIdentificationInLinearBandits} due to these differences.

\begin{algorithm}[t!]
    \caption{\algnameadaptive{} - Summary}
    \label{alg:adaptive_XY_allocation_summary}
    \begin{algorithmic}[1]
        \STATE \textbf{Input: }Set of arms $\calA$, subset size $K$, confidence $\delta > 0$, tuning parameters $t{'}$ and $\alpha$ 
        \STATE \textbf{Initialize: } $j \leftarrow 1$, $n_0 \leftarrow d(d + 1) + 1$, $\rho_0 \leftarrow 1$, $\tilde{\calA}_1 \leftarrow \calA$, and $\tilde{\calG}_1 \leftarrow \{\bx - \by : \bx, \by \in \tilde{\calA}_1\}$
        \WHILE[Stopping criterion]{$\abs{\tilde{\calA}_j} \neq 1$}
            \STATE \textbf{Initialize batch:} $t \leftarrow 1$, $V^{(0)} \leftarrow \mathbf{0}_{d \times d}$
            \STATE Randomly explore for $t{'}$ steps (Lines 3--7 in Algorithm \ref{alg:static_XY_allocation}).

            \WHILE{$\rho_j/t \geq \alpha \rho_{j - 1} / n_{j - 1}$}
                \STATE \COMMENT{Static strategy within a batch}
                \STATE Select the subset of $K$ arms (Lines 9--13 in Algorithm \ref{alg:static_XY_allocation}, but with $\calG$ replaced by $\tilde{\calG}_j$)
                \STATE $\rho_j = \max_{\bg \in \tilde{\calG}_j} \bg{'} {V\supt}^{-1} \bg$
            \ENDWHILE

            \STATE $n_j \leftarrow t$ \COMMENT{Prepare for the next batch}
            \STATE Estimate $\htheta^{(n_j)}$ from data collected in this batch
            \STATE $\tilde{\calA}_{j + 1} = \tilde{\calA}_j$
            \FOR[Eliminate arms]{$\ba_i \in \tilde{\calA}_{j}$}
                \IF{$\exists \ba_k \in \tilde{\calA}_j$ such that \eqref{eq:stopping_condition} holds for $\ba_k - \ba_i$}
                    \STATE $\tilde{\calA}_{j + 1} \leftarrow \tilde{\calA}_{j + 1} \backslash \{\ba_i\}$
                \ENDIF
            \ENDFOR
            \STATE $\tilde{\calG}_{j + 1} \leftarrow \{\bx - \by : \bx, \by \in \tilde{\calA}_{j + 1}\}$
            \STATE $j \leftarrow j + 1$
        \ENDWHILE
        \STATE \textbf{Return:} The arm in singleton set $\tilde{\calA}_j$
    \end{algorithmic}
\end{algorithm}


\section{Analysis}
\label{section:analysis}

\algname{} identifies the best-arm with probability at least $1 - \delta$ by the construction of the stopping criterion, as explained after eq. \eqref{eq:stopping_condition}. In this section, we prove an instance dependent upper bound on the sample complexity of \algname{}. We refer the reader to \cref{appendix:adaptive_strategy} for analogous theorems about \algnameadaptive. We also prove an instance-dependent lower bound on the sample complexity of any algorithm  that solves the linear-MNL-bandit problem, and prove that \algname is minimax optimal.

\subsection{Sample Complexity - Upper Bound}
\label{section:sample_complexity_upper_bound}

Recall from Section \ref{section:arm_selection_strategy} that our action-selection strategy greedily selects actions to satisfy the stopping criterion in \eqref{eq:stopping_condition}. In this section, we prove an instance dependent upper bound on the sample complexity of Algorithm \ref{alg:static_XY_allocation}.

\begin{theorem}
    \label{theorem:upper_bound}
    Using the stopping criterion from \eqref{eq:stopping_condition}, a $(1 + \beta)$-approximate action-selection strategy that solves \eqref{eq:arm_selection} satisfies
    \begin{align*}
        \rmP(\tau \leq \frac{512(1 + \beta)}{\kappa_\alpha^2 \Delta_{\min}^2} (d + &\log(3N^2 \tau^2/\delta)) \frac{d}{K} \,\, \\
        &\land \,\, \hat{\ba} = \ba^* ) \geq 1 - \delta,
    \end{align*}
    where $\hat{\ba}$ is the estimated best arm and $\tau$ is the number of time steps before the stopping criterion is satisfied.
\end{theorem}
\begin{proof}{(Sketch)}
    We only present a proof sketch here and refer the reader to Appendix \ref{appendix:sample_complexity_upper_bound} for details. We know that $\rmP(\hat{\ba} = \ba^*) \geq 1 - \delta$. In what follows, we condition on the event $\hat{\ba} = \ba^*$ and find an upper bound on $\tau$ that holds with probability $1$. Consider a further relaxation of eq. \eqref{eq:arm_selection},
    $$\Lambda^* = \argmin_{\Lambda \in \Delta_N} \max_{i, j \in [N]} \norm{\ba_i - \ba_j}[{V_\Lambda}^{-1}],$$ where $V_\Lambda = \sum_{i=1}^N \Lambda_i \ba_i \ba_i^\intercal$ and $\Delta_N$ is an $N$-dimensional simplex. Let $\hat{\Lambda}$ be the distribution over $N$ arms induced by $\{\bX\sups\}_{s \leq t}$, the optimal solution to the allocation problem in \eqref{eq:arm_selection}. It is easy to solve for $\Lambda^*$ but we want to solve for $\hat{\Lambda}$, an NP-hard problem, to ensure that every arm is pulled an integer number of times. There are efficient rounding procedures that first find $\Lambda^*$ and then round it to obtain an approximation to $\hat{\Lambda}$, denoted by $\tilde{\Lambda}$. Let $\rho(\Lambda) = \max_{i, j \in [N]} \norm{\ba_i - \ba_j}[{V_\Lambda}^{-1}][2]$ for any given $\Lambda$. Then, it can be shown that \cite{SoareEtAl:2014:BestArmIdentificationInLinearBandits}
    \begin{equation}
        \label{eq:rho_bound}
        \rho(\tilde{\Lambda}) \leq 2 (1 + \beta) d
    \end{equation}
    for some approximation factor $\beta > 0$.

    Recall that $\ba_1 = \ba^*$ by assumption. Because we condition on the event $\hat{\ba} = \ba^*$ and the algorithm terminates when the stopping criterion in \eqref{eq:stopping_condition} is satisfied, we have that for all $i \in [N]$ ,
    $$\frac{8}{\kappa_\alpha} \sqrt{d + \log(3N^2 \tau^2/\delta)} \norm{\ba^* - \ba_i}[{V^{(\tau)}}^{-1}] \leq \hDelta_{1i}^{(\tau)}.$$ Note that $\norm{\ba^* - \ba_i}[{V^{(\tau)}}^{-1}] = \frac{1}{\sqrt{\tau K}}\norm{\ba^* - \ba_i}[{V_{\tilde{\Lambda}}}^{-1}]$ as $V^{(\tau)} = \tau K V_{\tilde{\Lambda}}$. Because $\rho(\tilde{\Lambda}) \geq \norm{\ba^* - \ba_i}[{V_{\tilde{\Lambda}}}^{-1}][2]$ for any $i \in [N]$, the algorithm will have stopped if 
    \begin{equation}
        \label{eq:sufficient_stopping_criterion}
        \frac{64}{\kappa_\alpha^2} (d + \log(3N^2 \tau^2/\delta)) \frac{\rho(\tilde{\Lambda})}{\tau K} \leq (\hDelta_{1i}^{(\tau)})^2.
    \end{equation}
    Using \cref{corollary:confidence_bound}, we show in \cref{appendix:sample_complexity_upper_bound} that when eq. \eqref{eq:sufficient_stopping_criterion} holds, 
    $$\hDelta_{1i}^{(\tau)} \geq \frac{\Delta_{\min}}{2},$$
    where $\Delta_{\min} = \min_{i = 2, \dots, N} \mul{\bmtheta^*}{\ba^* - \ba_i}$.
    Thus an upper bound on the sample complexity is a $\tau$ that satisfies
    $$\frac{64}{\kappa_\alpha^2} (d + \log(3N^2 \tau^2/\delta)) \frac{\rho(\tilde{\Lambda})}{\tau K} = \frac{\Delta_{\min}^2}{4}.$$
    The desired result follows after rearranging terms and using eq. \eqref{eq:rho_bound}. 
\end{proof}


\subsection{Sample Complexity - Lower Bound}
\label{section:sample_complexity_lower_bound}

In this section, we derive an information-theoretic lower bound on the sample complexity of any algorithm that solves the linear-MNL-bandit. We consider a subclass of problems where $N = d$ and where $\ba_1, \dots, \ba_N$ are linearly independent but not necessarily orthogonal. Our lower bound matches the upper bound from Theorem \ref{theorem:upper_bound} on a worst-case problem instance.

The parameter $\bmtheta^* \in \bbR^d$ used in \eqref{eq:feedback_model} specifies a problem instance. Let $\tilde{\bmtheta} \in \bbR^d$ specify an alternate problem instance where $\ba_1$ is no longer the best arm (recall that $\ba_1$ is the best arm under $\bmtheta^*$). Let $E$ be the event that Algorithm \ref{alg:static_XY_allocation} returns $\ba_1$ as the best arm. Then, $P(E) \geq 1 - \delta$ under $\bmtheta^*$ and $P(E) \leq \delta$ under $\tilde{\bmtheta}$. The following change of measure lemma directly follows from Lemma 1 in \citet{KaufmannEtAl:2016:OnTheComplexityOfBestArmIdentificationInMultiArmedBanditModels}.

\begin{lemma}
    \label{lemma:change_of_measure}
    Let $\bmtheta^*$ and $\tilde{\bmtheta}$ be $d$-dimensional parameter vectors as specified above, and $N_S(\tau)$ be the number of times subset $S \subseteq_K \calA$ was chosen\footnote{Notational remark: $X \subseteq_K Y$ denotes that $X$ is a subset of $Y$ such that $\abs{X} = K$.} in the first $\tau$ time steps, where $\tau$ is an almost-surely finite stopping time. Also, let $\bmmu\supS(\bmtheta) \in \Delta_K$ denote the probability distribution over elements in $S$ under parameter $\bmtheta$ calculated using \eqref{eq:mu_def}. Then, $$\sumS \rmE_{\bmtheta^*}[N_S(\tau)] \, \rmKL{\bmmu\supS(\bmtheta^*)}{\bmmu\supS(\tilde{\bmtheta})} \geq \log \frac{1}{2.4 \delta}.$$
\end{lemma}

Note that the actions in our context correspond to selecting a subset of $K$ arms at each step. Thus, our setting is as if we have ${N \choose K}$ arms, each corresponding to a subset $S \subseteq_K \calA$. The reward for the action associated with a subset $S$ is drawn from the distribution $\bmmu\supS(\bmtheta)$. Hence, as opposed to Lemma 1 in \citet{KaufmannEtAl:2016:OnTheComplexityOfBestArmIdentificationInMultiArmedBanditModels}, the summation in our Lemma \ref{lemma:change_of_measure} runs over all subsets $S \subseteq_K \calA$. 

The challenge in deriving strong lower bounds lies in identifying an appropriate $\tilde{\bmtheta}$ that specifies an alternative problem instance. A common strategy in the classical MAB setting is to choose a $\tilde{\bmtheta}$ that changes the reward distribution of only one arm (i.e., one action), thus eliminating all but one term in the summation in Lemma \ref{lemma:change_of_measure} \cite{KaufmannEtAl:2016:OnTheComplexityOfBestArmIdentificationInMultiArmedBanditModels}. Doing so is harder under the MNL model because each arm is part of many subsets and affects the reward distribution of several actions (see the proof of Theorem \ref{theorem:lower_bound}). The challenge is exacerbated in linear bandits since a change in $\bmtheta$ changes the mean reward of multiple arms. \citet{FiezEtAl:2019:SequentialExperimentalDesignForTransductiveLinearBandits} obtain $\tilde{\bmtheta}$ by solving an optimization problem that makes a given arm $\ba_j \neq \ba^*$ the best arm while making the smallest perturbation to the original $\bmtheta^*$. However, unlike our lower bound in Theorem \ref{theorem:lower_bound}, the expression they derive does not explicitly show the dependence of sample complexity on parameters like $d$ and $K$.

Without loss of generality, assume that $\ba_1$ is the best arm under $\bmtheta^*$. Define $\bmtheta^j$ for $j = 2, \dots, d$ as,
\begin{align}
    \label{eq:theta_j_opt_problem}
    \bmtheta^j = \argmin_{\bmtheta \in \bbR^d} \,\,& \normtwosq{\bmtheta^* - \bmtheta} \notag \\
    \text{s.t.} \,\,& \mul{A_j}{\bmtheta^* - \bmtheta} = 0 \notag \\
    & \mul{\ba_1 - \ba_j}{\bmtheta^* - \bmtheta} \geq \epsilon + \Delta_{1j},
\end{align}
where $A_j \in \bbR^{d \times d - 1}$ contains $\ba_1, \dots, \ba_{j - 1}, \ba_{j + 1}, \dots, \ba_d$ as its columns, $\Delta_{1j} = \mul{\bmtheta^*}{\ba_1 - \ba_j}$, and $\epsilon > 0$ is a small constant. The equality constraint ensures that $\mul{\bmtheta^j}{\ba_i} = \mul{\bmtheta^*}{\ba_i}$ for all $i \in [d] \backslash \{j\}$, and the inequality constraint requires $\mul{\bmtheta^j}{\ba_j} \geq \mul{\bmtheta^j}{\ba_1} + \epsilon$. Hence, $\ba_1$ is no longer the best arm under parameter $\bmtheta^j$. Defining $F_j = I - A_j \mul{A_j}{A_j}^{-1} A_j^\intercal$, it is easy to see that the solution to \eqref{eq:theta_j_opt_problem} is given by $\bmtheta^j = \bmtheta^* - \delta^j$, where
\begin{equation}
    \label{eq:delta_j_def}
    \delta^j = \frac{\epsilon + \Delta_{1j}}{\norm{\ba_1 - \ba_j}[F_j][2]} F_j (\ba_1 - \ba_j).
\end{equation}
The following theorem uses Lemma \ref{lemma:change_of_measure} and an upper bound on $\rmKL{\bmmu\supS(\bmtheta^*)}{\bmmu\supS(\bmtheta^j)}$ for all $j = 2, \dots, d$.

\begin{restatable}{theorem}{lowerbound}
    \label{theorem:lower_bound}
    Let $N = d$ and $\ba_1, \dots, \ba_N \in \bbR^d$ span a $d$-dimensional subspace. Assume without loss of generality that $\mul{\bmtheta^*}{\ba_1} \geq \mul{\bmtheta^*}{\ba_i}$ for all $i = 2, \dots, N$. Define $\Delta_{1i} = \mul{\bmtheta^*}{\ba_1 - \ba_i}$ and let $\tau$ be the almost-surely finite stopping time before the stopping condition is satisfied. Then, for every $\epsilon > 0$ such that $\Delta_{1i} + \epsilon \leq 1$ for all $i \in [d] \backslash \{1\}$, $$\sumS \rmE_{\bmtheta^*} [N_S(\tau)] \geq \frac{1 - 1/K}{e} \sum_{j = 2}^d \frac{1}{(\Delta_{1j} + \epsilon)^2} \log\frac{1}{2.4 \delta},$$ where $\delta > 0$ is the error probability. 
\end{restatable}
\begin{proof} (Sketch)
    This is a brief proof sketch; see Appendix \ref{appendix:sample_complexity_lower_bound} for details. We overload the notation and use $S$ to denote both a set of arm vectors $\{\ba_{i_1}, \ba_{i_2}, \dots, \ba_{i_K}\} \subseteq_K \calA$ and the corresponding indices $\{i_1, i_2, \dots, i_K\}$. Further, $\mu_i\supS(\bmtheta)$ denotes the entry of $\bmmu\supS(\bmtheta)$ corresponding to the element $i \in S$. Using the constraints from the optimization problem in \eqref{eq:theta_j_opt_problem}, one can show that if $\ba_j \in S$, $$\rmKL{\bmmu\supS(\bmtheta^*)}{\bmmu\supS(\bmtheta^j)} = f_{\Delta_{1j} + \epsilon}(\mu_j\supS(\bmtheta^*)),$$ where $f_\alpha(x) \coloneqq \log (1 + x(\exp(\alpha) - 1)) - x\alpha$. Further, $f_\alpha(x) \leq f_\alpha(\bar{x})$ if $x \in [0, \bar{x}]$ for some $\bar{x} \leq \frac{1}{\alpha} - \frac{1}{\exp(\alpha) - 1}$. It can be shown that for large enough $K$,
    \begin{align*}
        \mu\supS_j(\bmtheta^*) \leq \frac{e}{K - 1} \leq 
        \frac{1}{\Delta_{1j} + \epsilon} - \frac{1}{\exp(\Delta_{1j} + \epsilon) - 1}.
    \end{align*}
    Thus, $f_{\Delta_{1j} + \epsilon}(\mu_j\supS(\bmtheta^*)) \leq  f_{\Delta_{1j} + \epsilon}(\frac{e}{K - 1})$. This provides an upper bound on $\rmKL{\bmmu\supS(\bmtheta^*)}{\bmmu\supS(\bmtheta^j)}$. Using this bound in Lemma \ref{lemma:change_of_measure} gives a lower bound on $\sum_{j \in S}\rmE_{\bmtheta^*}[N_S(\tau)]$. Summing over $j = 2, \dots, d$ and dividing by $K$ to account for the double-counting yields the desired result.
\end{proof}

\begin{remark}
    \label{remark:worst_case_optimality}
    Note that $K \leq N$ in general. Hence, $K \leq d$ for the problem instance in Theorem \ref{theorem:lower_bound}. Consider the case where $K = N = d$ and $\Delta_{1j} = \Delta_{\min}$ for all $j = 2, 3, \dots, N$. This results in an $\Omega(\frac{d}{\Delta_{\min}^2})$ lower bound using Theorem \ref{theorem:lower_bound} which matches the $\tilde{O}(\frac{d^2}{K \Delta_{\min}^2})$ upper bound from Theorem \ref{theorem:upper_bound} up to logarithmic factors.
\end{remark}


\section{Experiments}
\label{section:experiments}

We perform four types of experiments. First, we study the dependence of the stopping time on $d$ and $K$ and verify that it matches the predictions of our upper bound. Second, we study the arm-pulls profile of \algname{} and \algnameadaptive{}. Third, we test the robustness of our algorithms by using a feedback model different from the MNL feedback model (eq. \eqref{eq:feedback_model}). Finally, we compare \algname{} and \algnameadaptive{} with the fully adaptive allocation strategy in \citet{KazerouniEtAl:2019:BestArmIdentificationInGeneralizedLinearBandits} for the case when $K = 2$, as this is the only case when their algorithm can be used.

Throughout this section, we consider a problem setting where $N = d + 1$ and $\ba_i = \be_i \in \bbR^d$ for all $i \in [d]$. Here, $\be_i$ is the $i^{th}$ standard basis vector. The $(d + 1)^{th}$ arm vector is given by $\ba_{d + 1} = [\cos{\omega}, \sin{\omega}, 0, 0, \dots, 0]$ for $\omega = 0.01$. We set $\bmtheta^* = [2, 0, 0, \dots, 0]$, making $\ba_1$ the best arm and $\ba_{d + 1}$ a close second-best arm. This is the setting studied by most papers on best arm identification in the linear setting \cite{SoareEtAl:2014:BestArmIdentificationInLinearBandits,XuEtAl:2018:AFullyAdaptiveAlgorithmForPureExplorationInLinearBandits,FiezEtAl:2019:SequentialExperimentalDesignForTransductiveLinearBandits}. 


\begin{figure}[t]
    \centering
    \subfloat[][]{
        \label{fig:stopping_time_synthetic}
        \includegraphics[width=0.23\textwidth]{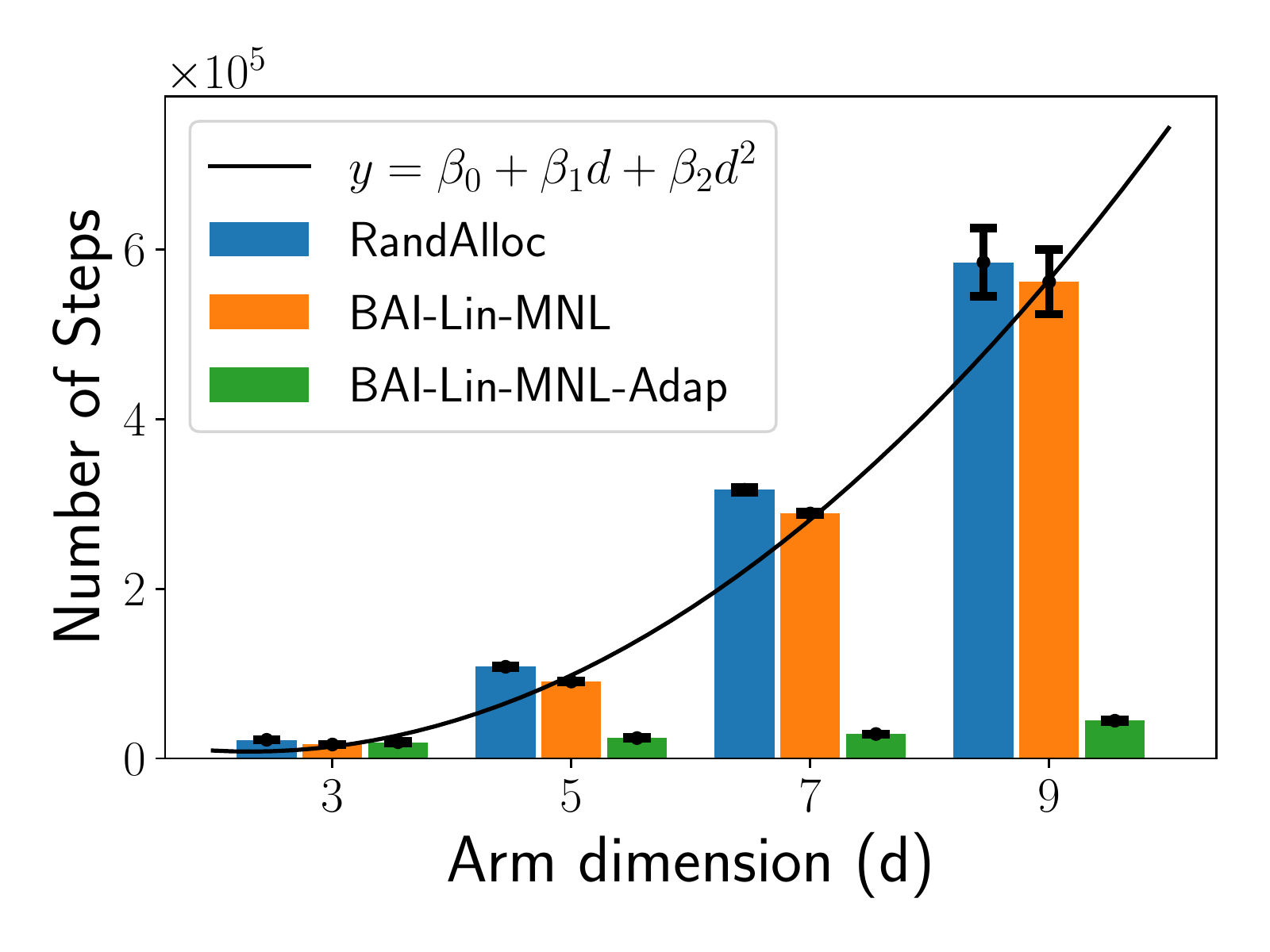}
    }%
    \subfloat[][]{
        \label{fig:K_ablation}
        \includegraphics[width=0.23\textwidth]{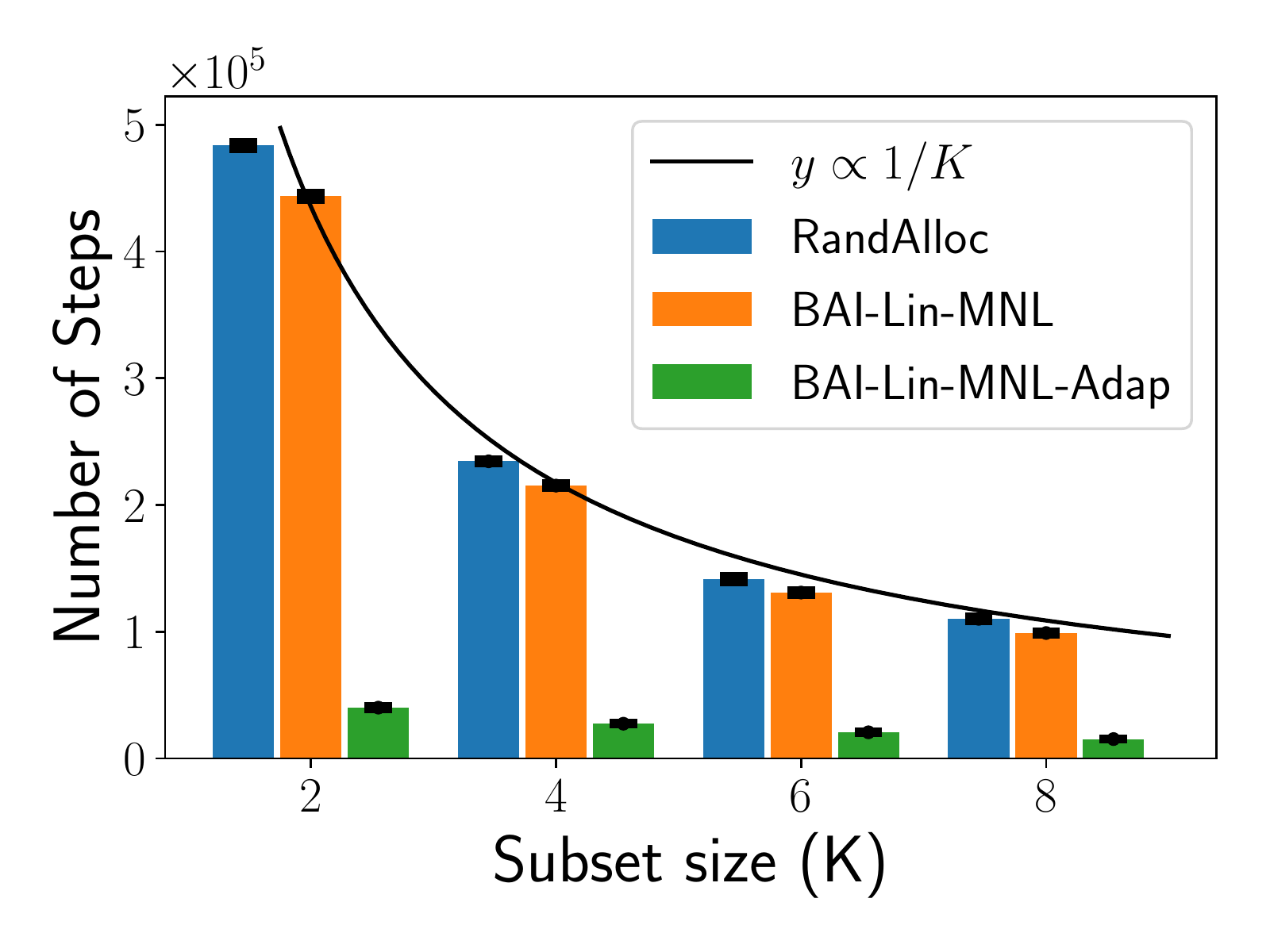}
    }
    \caption{Variation of mean stopping time with arm-dimension $d$ and subset-size $K$. Plot \subref{fig:stopping_time_synthetic} uses $K = 3$ and plot \subref{fig:K_ablation} uses $d = 7$. The stopping time of \algname{} increases as $d^2$ and decreases as $1/K$ as predicted by Theorem \ref{theorem:upper_bound}. \algnameadaptive{} performs significantly better than the other two strategies.}
\end{figure}

\subsection{Sample Complexity Dependence on $d$ and $K$}
\label{section:variation_of_sample_complexity}

We study the stopping time dependence on arm-dimension $d$ and subset-size $K$ for three algorithms: RandAlloc (a random allocation strategy that selects actions randomly), \algname, and \algnameadaptive. Figures \ref{fig:stopping_time_synthetic} and \ref{fig:K_ablation} show the variation in stopping time as a function of $d$ and $K$ respectively. Each strategy was independently run $10$ times. The plots validate Theorem \ref{theorem:upper_bound} which predicts that the sample complexity of \algname{} increases as $d^2$ and decreases as $1/K$. We also see that \algnameadaptive{} significantly outperforms the other two strategies (up to 12x fewer samples). Note that static allocation strategies do not perform as well, even in the linear bandits case \cite{SoareEtAl:2014:BestArmIdentificationInLinearBandits,XuEtAl:2018:AFullyAdaptiveAlgorithmForPureExplorationInLinearBandits}. Ours is the first algorithm for best-arm identification under MNL feedback in the linear bandits setting, and paves way for better algorithms and tighter analysis in future.


\begin{figure}
    \centering
    \includegraphics[width=0.32\textwidth]{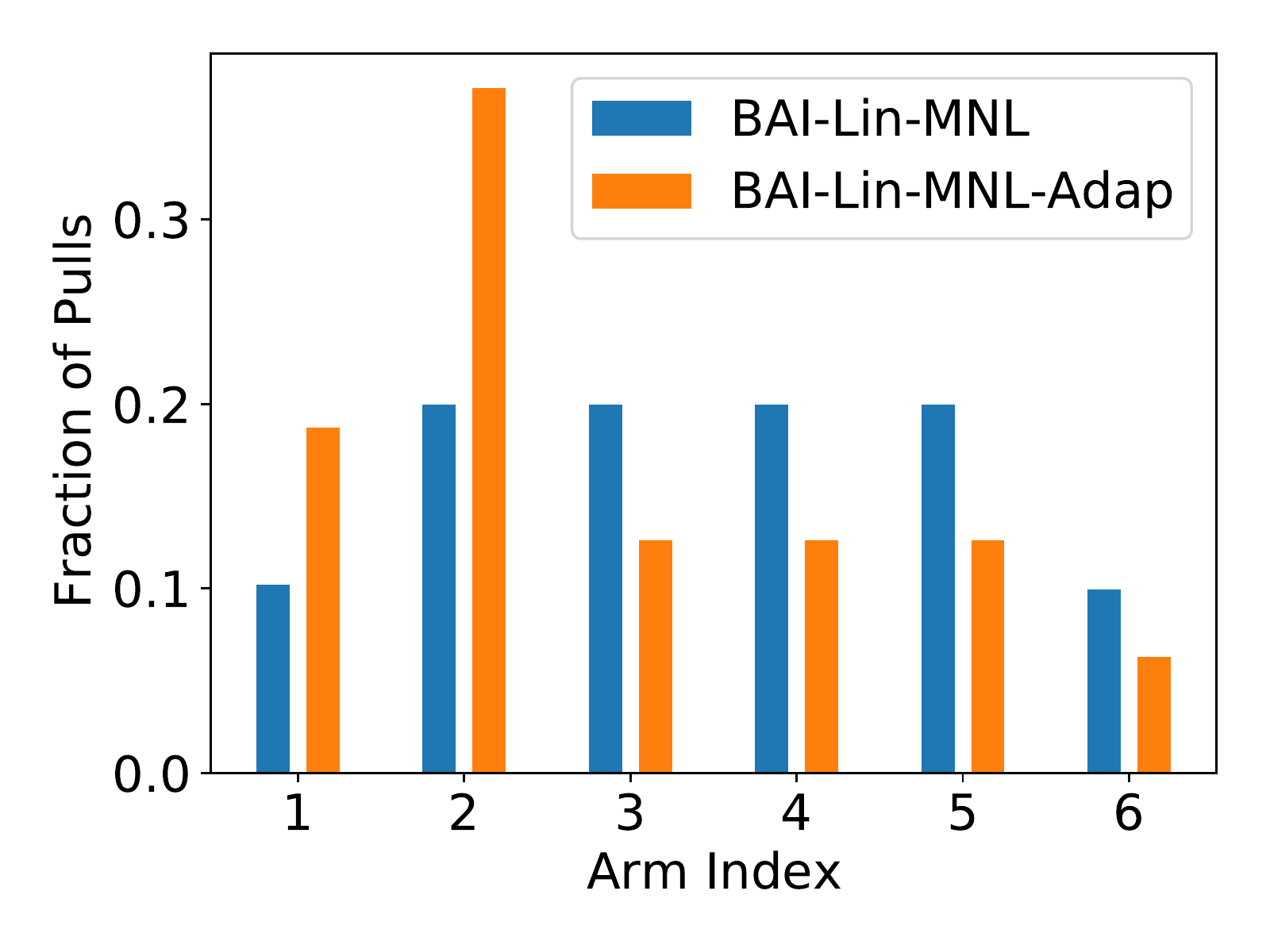}
    \caption{Fraction of times each arm was pulled by \algname{} and \algnameadaptive{}. \algnameadaptive{} pulls arm $\ba_2$ more often as it helps in differentiating $\ba_1$ from $\ba_6$. $d=5$ and $K=3$}
    \label{fig:arm_counts}
    \vskip -5mm
\end{figure}

\subsection{Profile of Arm Pulls}
\label{section:static_vs_adaptive}
Arms $\ba_1$ and $\ba_{d + 1}$ are the top two arms, and hence the estimate of $\bmtheta^*$ must be improved along $\ba_1 - \ba_{d + 1}$ to differentiate between these arms. A static allocation strategy such as \algname{} explores all directions uniformly. On the other hand, \algnameadaptive{} eliminates unimportant directions through successive batches. We verify this behavior in Figure \ref{fig:arm_counts} which shows the fraction of times each arm was selected by \algname{} and \algnameadaptive{} when $d=5$ and $K=3$. We see that $\ba_2$ is selected more often by \algnameadaptive{}, as arm $\ba_2$ is most aligned with $\ba_1 - \ba_{d + 1}$ among all arms.


\subsection{Robustness}
\label{section:robustness}

Our analysis assumes that the winner is chosen according to eq.\eqref{eq:feedback_model} at each step. However, our algorithms can be applied even when the winner is chosen according to a different model. The MNL feedback model in \eqref{eq:feedback_model} is an instance of a class of choice models known as Random Utility Models (RUM) \cite{AzariEtAl:2012:RandomUtilityTheoryForSocialChoice,SoufianiEtAl:2013:GeneralizedRandomUtilityModelsWithMultipleTypes}. We experimented with another RUM where the winner at each step is chosen as $\argmax_{\ba \in S} (\mul{\bmtheta^*}{\ba} + \eta_\ba)$, where $\eta_\ba \sim \calN(0, \sigma^2)$ are chosen i.i.d. for some constant $\sigma > 0$ which we set to $1.0$ in our experiment. Table \ref{table:alternative_feedback_model} compares the performance of RandAlloc, \algname{}, and \algnameadaptive{} for $K = 3$ and $d = 7$. Once again, \algnameadaptive{} outperforms both strategies while also returning the correct best-arm.

\begin{table}
    \centering
    \begin{tabular}{l r@{$\pm$}l}
        \toprule
        \textbf{Strategy} & \multicolumn{2}{c}{\textbf{Stopping time}} \\
        \midrule
        RandAlloc & $181370$ & $1085$ \\
        \algname{} & $166761$ &  $893$ \\
        \algnameadaptive{} & $\mathbf{23916}$ & $\mathbf{713}$ \\
        \bottomrule
    \end{tabular}
    \caption{Robustness: Stopping time of various strategies under a different Random Utility Model described in Section \ref{section:robustness}.}
    \label{table:alternative_feedback_model}
\end{table}


\begin{figure}
    \centering
    \includegraphics[width=0.37\textwidth]{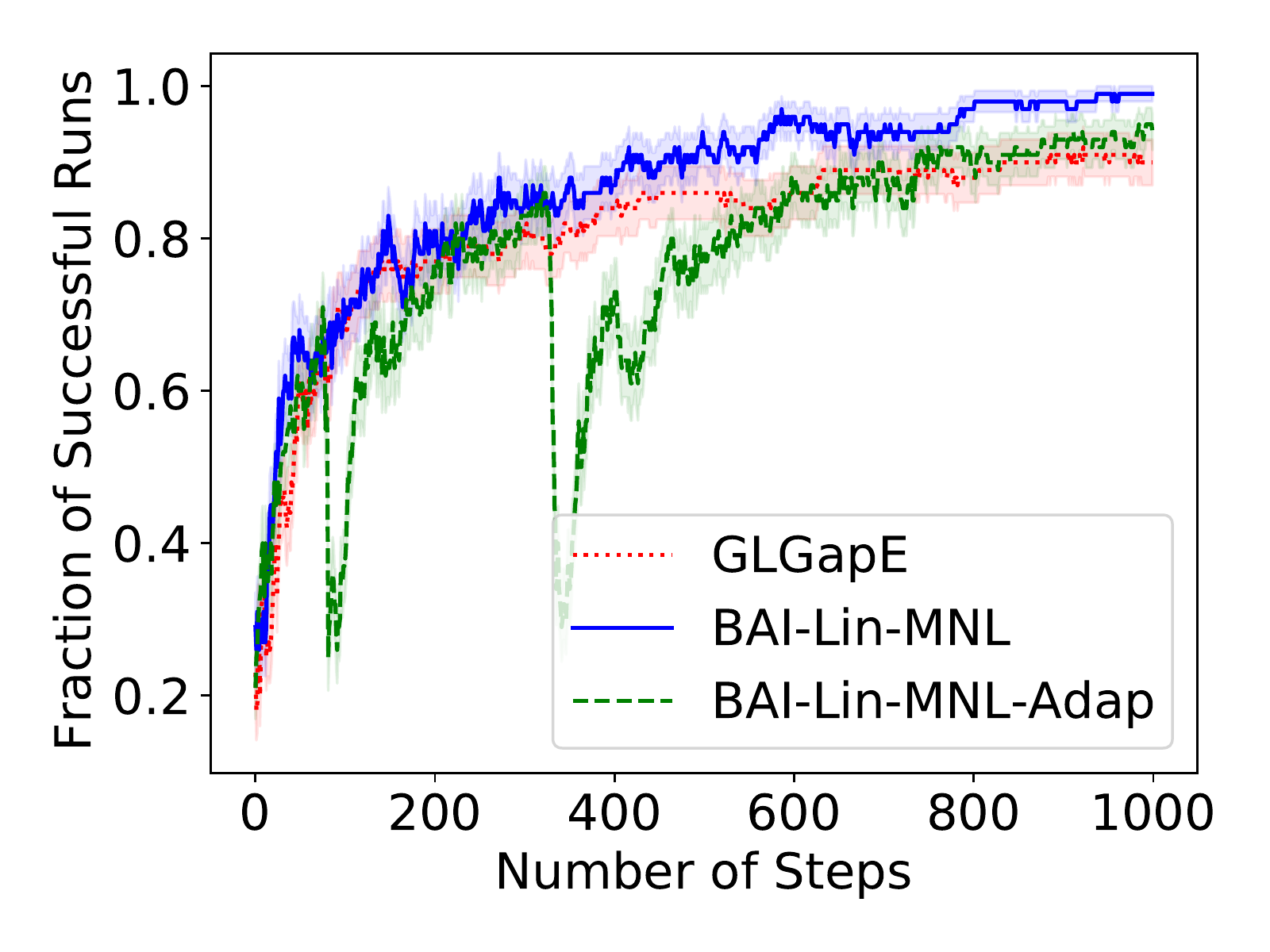}
    \caption{Comparing \algname{} with GLGapE \cite{KazerouniEtAl:2019:BestArmIdentificationInGeneralizedLinearBandits}. $d = 8$, $K = 2$.}
    \label{fig:comparison_with_GLGapE}
    \vskip -5mm
\end{figure}

\subsection{Comparison with a Fully Adaptive Strategy}
\label{section:comparison_with_a_fully_adaptive_strategy}

\algnameadaptive{} must discard the data from previous batches to ensure that the $\{\bX\sups\}$ sequence within a batch is a fixed non-adaptive sequence and the assumption in Corollary \ref{corollary:confidence_bound} is satisfied. To avoid discarding data, approaches that use confidence bounds for adaptive sequences $\{\bX\sups\}_{s \leq t}$ have been proposed \cite{XuEtAl:2018:AFullyAdaptiveAlgorithmForPureExplorationInLinearBandits,KazerouniEtAl:2019:BestArmIdentificationInGeneralizedLinearBandits}. While \citet{XuEtAl:2018:AFullyAdaptiveAlgorithmForPureExplorationInLinearBandits} study linear bandits and their algorithm cannot be used in our setting, \citet{KazerouniEtAl:2019:BestArmIdentificationInGeneralizedLinearBandits} study best-arm identification when the feedback is generated according to a generalized linear model, and this feedback can be simulated using our setting when $K=2$ as we explain next. Given arms $\calA=\{\ba_i\}_{i=1}^N$ in the linear-MNL setting, define arms $\bb_{ij} \in \bbR^d$ as $\bb_{ij} = \ba_i - \ba_j$ for all $i, j \in [N]$ in the generalized linear setting. The feedback for arm $\bb_{ij}$ in the generalized linear setting is $1$ with probability $\sigma(\mul{\bmtheta^*}{\bb_{ij}})$ and $0$ otherwise, where $\sigma(x) = 1 / (1 + \exp(-x))$ is the logistic sigmoid function. This feedback can be simulated by playing the subset $\{\ba_i, \ba_j\}$ in the linear-MNL model and drawing the winner according to eq. \eqref{eq:feedback_model}.

We compare the performance (over $100$ simulations) of our algorithms to the algorithm GLGapE proposed by \citet{KazerouniEtAl:2019:BestArmIdentificationInGeneralizedLinearBandits} in \cref{fig:comparison_with_GLGapE}. At each step $t$, we use the estimated $\htheta\supt$ to identify the best-arm, and plot the fraction of simulations that correctly estimated the best arm at any given time $t$. Both \algname{} and \algnameadaptive{} identify the best arm with the same probability or higher than GLGapE, although not significantly for \algnameadaptive. The fluctuations in \algnameadaptive{} correspond to batch resets where all previous data is discarded.

\vspace{-0.3cm}
\paragraph{Implementation notes:} While our theoretical results do not consider regularization while computing the maximum likelihood estimate $\htheta$, we use it in our experiments with regularization coefficient set to $\lambda = 10^{-4}$. Moreover, as is the common practice \cite{LiEtAl:2017:ProvablyOptimalAlgorithmsForGeneralizedLinearContextualBandits,KazerouniEtAl:2019:BestArmIdentificationInGeneralizedLinearBandits}, we ignore the condition on $\lambdamin{V^{(t)}}$ required by Corollary \ref{corollary:confidence_bound} in our implementation, and execute a fixed number of random exploration steps (set of $5$). We also add $\mu I$ ($\mu = 10^{-4}$) to $V^{(t - 1)}$ in Line 10 in Algorithm \ref{alg:static_XY_allocation} to ensure that it is invertible. Because any subset $S \subseteq_K \calA$ in which all the $K$ arms are same does not provide any information under the MNL feedback model, we discard such subsets for all strategies and replace them by the second best solution in Line 10. This ensures that at least one arm in each selected subset is different. Finally, we set $\kappa_\alpha = 0.5$ without tuning and use $\delta = 0.05$.


\section{Conclusion}
\label{section:conclusion}

In this paper, we study the problem of best-arm identification under structured preference feedback. We derive a confidence bound for the unknown parameter $\bmtheta^*$ under the MNL feedback model, develop static and adaptive algorithms, analyze their sample complexity, and prove that they are minimax optimal. To the best of our knowledge, this is the first work that studies best arm identification under structured preference feedback. Devising a fully adaptive strategy in this setting is a promising direction for future work. Another interesting problem is bridging the gap between the upper and lower bounds in the general case. We hypothesize that this can be achieved by improving the confidence bound in Theorem \ref{theorem:confidence_interval}. 


\bibliography{biblio}
\bibliographystyle{icml2020}


\onecolumn
\appendix

\begin{center}
  \Large{\textsc{Supplementary Material}}\textbf{\Large{\\ \vspace{0.2cm} Pure Exploration with Structured Preference Feedback}} \\
  \line(1, 0){450}
\end{center}

\section{Maximum Likelihood Estimation}
\label{appendix:mle}

Let $\mathcal{D}\supt = \{(\bX\sups, \by\sups)\}_{s \in [t]}$ be the set of observations till time $t$. The maximum likelihood estimate $\htheta \supt$ at time $t$ is given by:
\begin{equation*}
    \htheta\supt = \argmax_{\bmtheta \in \bbR^d} \ell(\bmtheta; \mathcal{D}\supt),
\end{equation*}
where $\ell(\bmtheta, \mathcal{D}\supt)$ is the log-likelihood function defined as:
\begin{equation*}
    \ell(\bmtheta, \mathcal{D}\supt) = \sum_{s = 1}^t \sum_{i = 1}^K y_i\sups \log \mu_i\sups(\bmtheta),
\end{equation*}
where $\mu_i\sups(\bmtheta)$ is defined in \eqref{eq:mu_def}. The derivative of $\ell$ with respect to $\bmtheta$ is given by
\begin{equation*}
    \nabla_{\bmtheta} \ell = \sums \sum_{i = 1}^K y_i \sups (\bx_i\sups - \bX\sups \bmmu\sups(\bmtheta)) = \sums \bX\sups (\by\sups - \bmmu\sups(\bmtheta)).
\end{equation*}
The maximum likelihood solution $\htheta\supt$ satisfies $[\nabla_{\bmtheta} \ell]_{\bmtheta = \htheta\supt} = 0$ as it maximizes $\ell(\bmtheta, \mathcal{D}\supt)$.


\section{Confidence Bound}
\label{appendix:confidence_bound}

Recall Theorem \ref{theorem:confidence_interval}:

\confidencebound*

\begin{proof}
We prove Theorem \ref{theorem:confidence_interval} by a series of technical lemmas. The proofs for these lemmas are given in Appendix \ref{appendix:confidence_bound_technical_lemmas}. We will use $\htheta \coloneqq \htheta\supt$ for the remainder of this section. Recall that the objective is to show a high probability bound on $\abs{\mul{\bx}{\htheta - \bmtheta^*}}$ for any $\bx \in \bbR^d$. Define the error function $G: \bbR^d \rightarrow \bbR^d$ as:
\begin{equation*}
    G(\bmtheta) = \sums \bX\sups (\bmmu\sups(\bmtheta) - \bmmu\sups(\bmtheta^*)).
\end{equation*}
Note that $G(\bmtheta^*) = 0$ and $G(\htheta)$ is given by:
\begin{equation*}
    G(\htheta) = \sums \bX\sups \bmepsilon\sups,
\end{equation*}
where, we use the fact that $[\nabla_{\bmtheta} \ell]_{\bmtheta = \htheta} = 0$ and write the observed output $\by\sups$ as $\by\sups = \bmmu\sups(\bmtheta^*) + \bmepsilon\sups$. Note that $\bmepsilon\sups \in [0, 1]^K$ and $\epsilon\sups_i = -\mu_i\sups(\bmtheta^*)$ if $y_i\sups = 0$ and $\epsilon\sups_i = 1 - \mu\sups_i({\bmtheta^*})$ otherwise. It is easy to see that $\rmE[\epsilon_i\sups] = 0$ for all $s \in [t]$ and  $i \in [K]$. Using the mean-value theorem for vector valued functions\footnote{\href{https://en.wikipedia.org/wiki/Mean_value_theorem}{Wikipedia article: \texttt{https://en.wikipedia.org/wiki/Mean\_value\_theorem}}} we get:
\begin{equation}
    \label{eq:mvt_G}
    G(\bmtheta_1) - G(\bmtheta_2) = \Big[ \int_0^1 [\nabla_{\bmtheta} G]_{\bmtheta = q\bmtheta_1 + (1 - q)\bmtheta_2} \mathrm{dq} \Big] (\bmtheta_1 - \bmtheta_2).
\end{equation}
Using the chain rule of derivatives, we can compute $\nabla_{\bmtheta} G$ as follows
\begin{align*}
    \nabla_{\bmtheta} G &= \nabla_{\bmtheta} \sums \Big[\bX\sups \bmmu\sups(\bmtheta) \Big] \notag \\
    &= \sums \Big[ \nabla_{\bmtheta} \bmmu\sups(\bmtheta) \Big] \bX\sups{'} \notag \\
    &= \sums \bX\sups \Big[ \mathrm{diag}(\bmmu\sups(\bmtheta)) - \bmmu\sups(\bmtheta)\bmmu\sups(\bmtheta){'} \Big] \bX\sups{'}.
\end{align*}
Recall the definition of $F(\bmtheta_1, \bmtheta_2)$ from \eqref{eq:F_def}, $\bM\sups(\bmtheta_1, \bmtheta_2)$ from \eqref{eq:M_def}, and $B_\alpha$, $\kappa_\alpha$, and $V\supt$ from Assumption \ref{assumption:kappa_def}. It is easy to see that $F(\bmtheta_1, \bmtheta_2) = \int_0^1 [\nabla_{\bmtheta} G]_{\bmtheta = q\bmtheta_1 + (1 - q)\bmtheta_2} \mathrm{dq}$. The following lemma describes some useful properties of $F$. We will abbreviate $V\supt$ by $V$ for the remainder of this section.

\begin{lemma}
    \label{lemma:lambdamin_F}
    The following relations hold for all $\bmtheta \in B_\alpha$:
    \begin{enumerate}
        \item $\lambdamin{F(\bmtheta, \bmtheta^*)} \geq \kappa_\alpha \lambdamin{V}$
        \item $\lambdamin{F(\bmtheta, \bmtheta^*)^{-1}} \leq \frac{1}{\kappa_\alpha} \lambdamin{V^{-1}}$
        \item $\lambdamin{F(\bmtheta, \bmtheta^*) V^{-1} F(\bmtheta, \bmtheta^*)} \geq \kappa_\alpha^2 \lambdamin{V}$
    \end{enumerate}
\end{lemma}
Using Lemma \ref{lemma:lambdamin_F}, for every $\bmtheta \in B_\alpha$, the following holds:
\begin{align}
    \label{eq:G_theta_thetastar_lower_bound}
    \norm{G(\bmtheta)}[V^{-1}][2] &= \norm{G(\bmtheta) - G(\bmtheta^*)}[V^{-1}][2] \notag \\
    &=(\bmtheta - \bmtheta^*)' F(\bmtheta, \bmtheta^*) V^{-1} F(\bmtheta, \bmtheta^*) (\bmtheta - \bmtheta^*) \notag \\
    &\geq \lambdamin{F(\bmtheta, \bmtheta^*) V^{-1} F(\bmtheta, \bmtheta^*)} \norm{\bmtheta - \bmtheta^*}[2][2] \notag \\
    &\geq \kappa_\alpha^2 \lambdamin{V} \norm{\bmtheta - \bmtheta^*}[2][2].
\end{align}
Assuming $\htheta \in B_\alpha$ (we will find a suitable $\alpha$ for which this is true later), we get as a special case of \eqref{eq:G_theta_thetastar_lower_bound},
\begin{equation}
    \label{eq:G_theta_hat_lower_bound}
    \norm{G(\htheta)}[V^{-1}] \geq \kappa_\alpha \sqrt{\lambdamin{V}} \norm{\htheta - \bmtheta^*}[2].
\end{equation}

Assume that $\lambdamin{F(\bmtheta_1, \bmtheta_2)} > 0$ for all $\theta_1 \neq \theta_2$, $\theta_1, \theta_2 \in B_\alpha$. Then, $(\bmtheta_1 - \bmtheta_2){'} (G(\bmtheta_1) - G(\bmtheta_2)) = (\bmtheta_1 - \bmtheta_2){'} F(\bmtheta_1, \bmtheta_2)(\bmtheta_1 - \bmtheta_2) > 0$. Thus, $G(\bmtheta_1) - G(\bmtheta_2) = 0$ if and only if $\bmtheta_1 = \bmtheta_2$.

\begin{lemma}{(Lemma A in \cite{ChenEtAl:1999:StrongConsistencyOfMLEInGLMWithFixedAndAdaptiveDesigns})}
    \label{lemma:chen_lemma_A}
    Let $G$ be a smooth injection from $\bbR^d \rightarrow \bbR^d$ with $G(\bmtheta^*) = \mathbf{0}$. Let $\partial B_\alpha = \{\bmtheta \in \bbR^d : \norm{\bmtheta - \bmtheta^*}[2] = \alpha\}$, then $\inf_{\bmtheta \in \partial B_\alpha} \norm{G(\bmtheta)}[V^{-1}] \geq r$ implies that $\{\bmtheta : \norm{G(\bmtheta)}[V^{-1}] \leq r\} \subseteq B_\alpha$.
\end{lemma}

Lemma \ref{lemma:chen_lemma_A} applies as $G$ is an injective function. Moreover, $\inf_{\bmtheta \in \partial B_\alpha}\norm{G(\bmtheta)}[V^{-1}] \geq \kappa_\alpha \alpha \sqrt{\lambdamin{V}}$ using \eqref{eq:G_theta_hat_lower_bound}. Thus, from Lemma \ref{lemma:chen_lemma_A}, $\{\bmtheta : \norm{G(\bmtheta)}[V^{-1}] \leq \kappa_\alpha \alpha \sqrt{\lambdamin{V}}\} \subseteq B_\alpha$. If we can find a large enough $\alpha$ such that $\norm{G(\htheta)}[V^{-1}] \leq \kappa_\alpha \alpha \sqrt{\lambdamin{V}}$, then $\htheta \in B_\alpha$, and hence \eqref{eq:G_theta_hat_lower_bound} will hold.

\begin{lemma}
    \label{lemma:high_prob_alpha_bound}
    Assume that the feature vectors satisfy $\norm{\ba_i}[2] \leq 1$ for all $i \in [N]$. Event $\mathcal{E}_G \coloneqq \{\norm{G(\htheta)}[V^{-1}] \leq 4 \sqrt{d + \log (1 / \delta)}\}$ happens with probability $\geq 1 - \delta$.
\end{lemma}

Using Lemma \ref{lemma:high_prob_alpha_bound}, setting $\alpha \geq \frac{4}{\kappa_\alpha} \sqrt{\frac{d + \log(1/\delta)}{\lambdamin{V}}}$ ensures that $\norm{G(\htheta)}[V^{-1}] \leq \kappa_\alpha \alpha \sqrt{\lambdamin{V}}$ with probability $\geq 1 - \delta$. Thus, $\htheta \in B_\alpha$ and hence \eqref{eq:G_theta_hat_lower_bound} holds. Rearranging \eqref{eq:G_theta_hat_lower_bound}, we get
\begin{align}
    \label{eq:htheta_close_to_optimal}
    \norm{\htheta - \bmtheta^*}[2] &\leq \frac{1}{\kappa_\alpha \sqrt{\lambdamin{V}}} \norm{G(\htheta)}[V^{-1}] \notag \\
    &\leq \frac{4}{\kappa_\alpha} \sqrt{\frac{d + \log(1/\delta)}{\lambdamin{V}}} \notag \\
    &\leq 1.
\end{align}
Here, the last line assumes that $\lambdamin{V} \geq 16 (d + \log (1 / \delta)) / \kappa^2$, where $\kappa \coloneqq \kappa_1$. Define $\Delta = \htheta - \bmtheta^*$ and $Z = G(\htheta) = G(\htheta) - G(\bmtheta^*)$. We have:
\begin{equation*}
    Z = G(\htheta) - G(\bmtheta^*) = F(\htheta, \bmtheta^*) \Delta = (H + E) \Delta,
\end{equation*}
where $H = F(\bmtheta^*, \bmtheta^*)$ and $E = F(\htheta, \bmtheta^*) - F(\bmtheta^*, \bmtheta^*)$. From this, we can compute $\Delta = (H + E)^{-1} Z$. Using the identity $(H + E)^{-1} = H^{-1} - H^{-1}E(H + E)^{-1}$, 
\begin{align}
    \label{eq:step2}
    \abs{\mul{\mathbf{x}}{\Delta}} = \abs{\mul{\mathbf{x}}{(H + E)^{-1}Z}} &= \abs{\mul{\mathbf{x}}{H^{-1}Z} - \mul{\mathbf{x}}{H^{-1}E(H + E)^{-1}Z}} \notag \\
    &\leq \abs{\mul{\mathbf{x}}{H^{-1}Z}} + \abs{\mul{\mathbf{x}}{H^{-1}E(H + E)^{-1}Z}}
\end{align}

\begin{lemma}
    \label{lemma:term1_bound}
    Assume that feature vectors satisfy $\norm{\ba_i}[2] \leq 1$ for all $i \in [N]$. Define $\kappa^* = \sup \{\kappa \in \bbR^d : F(\bmtheta^*, \bmtheta^*) \succeq \kappa V\}$. Note that $\kappa_* \geq \kappa_\alpha > 0$ where the inequality follows from Assumption \ref{assumption:kappa_def}. Then, with probability at least $1 - 2\delta$:
    $$\abs{\mul{\bx}{H^{-1}Z}} \leq  \frac{2}{\kappa^*} \sqrt{2 \log(1/\delta)} \norm{\bx}[V^{-1}] \leq \frac{2}{\kappa_\alpha} \sqrt{2 \log(1/\delta)} \norm{\bx}[V^{-1}].$$
\end{lemma}

To bound the second term in \eqref{eq:step2}, we begin by applying Cauchy-Schwarz:
\begin{equation*}
    \mul{\bx}{H^{-1}E(H + E)^{-1}Z} \leq \norm{\bx}[H^{-1}] \,\, \norm{H^{-1/2}E(H + E)^{-1} H^{1/2}} \,\, \norm{Z}[H^{-1}].
\end{equation*}
Note that $\norm{\bx}[H^{-1}] \leq \frac{1}{\sqrt{\kappa^*}} \norm{\bx}[V^{-1}]$ (see proof of Lemma \ref{lemma:term1_bound}). Similarly, $\norm{Z}[H^{-1}] \leq \frac{1}{\sqrt{\kappa^*}} \norm{Z}[V^{-1}]$. Thus,
\begin{equation}
    \label{eq:step2_second_term_breakdown}
    \mul{\bx}{H^{-1}E(H + E)^{-1}Z} \leq \frac{1}{\kappa^*}\norm{\bx}[V^{-1}] \,\, \norm{H^{-1/2}E(H + E)^{-1} H^{1/2}} \,\, \norm{Z}[V^{-1}].
\end{equation}
To bound the second term in the inequality above, we again use the identity $(H + E)^{-1} = H^{-1} - H^{-1}E(H + E)^{-1}$,
\begin{align*}
    \norm{H^{-1/2}E(H + E)^{-1} H^{1/2}} &=  \norm{H^{-1/2}E(H^{-1} - H^{-1}E(H + E)^{-1}) H^{1/2}} \notag \\
    &= \norm{H^{-1/2}E H^{-1/2} - H^{-1/2}E H^{-1}E(H + E)^{-1} H^{1/2}} \notag \\
    &\leq \norm{H^{-1/2}E H^{-1/2}} + \norm{H^{-1/2}E H^{-1/2}} \norm{H^{-1/2}E(H + E)^{-1} H^{1/2}}. 
\end{align*}
Thus,
\begin{align}
    \label{eq:step2_second_term_middle_term_breakdown}
    \norm{H^{-1/2}E(H + E)^{-1} H^{1/2}} \leq \frac{\norm{H^{-1/2}E H^{-1/2}}}{1 - \norm{H^{-1/2}E H^{-1/2}}} \leq 2 \norm{H^{-1/2}E H^{-1/2}},
\end{align}
where, the second inequality follows from $\frac{x}{1 - x} \leq 2x$ if $x \in [0, 0.5]$. Using the definition of $F$ from \eqref{eq:F_def}, we get:
\begin{align*}
    E &= F(\htheta, \bmtheta^*) - F(\bmtheta^*, \bmtheta^*) \notag \\
    &= \sums \bX\sups \Big(\bM\sups (\htheta, \bmtheta^*) - \bM\sups (\bmtheta^*, \bmtheta^*) \Big) \bX\sups{'}.
\end{align*}
Recall the definition of $f\sups$, $S\sups$, $S\sups(i)$, and $\tilde{\lambda}$ from Assumption \ref{assumption:lambda_tilde_def}. We can write $E = \sums \bX\sups (f\sups(\htheta) - f\sups(\bmtheta^*)) \bX\sups{'}$. Using mean-value theorem for vector-valued functions on $f\sups$, we get:
\begin{align*}
    f\sups(\htheta) - f\sups(\bmtheta^*) &= \Big\{ \int_0^1 [\nabla_{\bmtheta} f\sups]_{\bmtheta = q\htheta + (1 - q)\bmtheta^*} \mathrm{dq} \Big\} \odot \Delta \notag \\
    &= \sum_{i = 1}^d (\htheta_i - \bmtheta^*_i) S\sups(i).
\end{align*}
Note that $\nabla_{\bmtheta} f\sups$ is a $K \times K \times d$ tensor and $\odot$ operator perform dot product along the third dimension of this tensor. Now, $E = \sums \sum_{i = 1}^d (\htheta_i - \bmtheta^*_i) \bX\sups S\sups(i) \bX\sups{'}$. To find $\norm{H^{-1/2}E H^{-1/2}}$, we write
\begin{align*}
    \mul{\bx}{H^{-1/2} E H^{-1/2} \bx} &= \sums \sum_{i = 1}^d (\htheta_i - \bmtheta^*_i) \bx{'} H^{-1/2} \bX\sups S\sups(i) \bX\sups{'} H^{-1/2} \bx \notag \\
    &\leq \sums \sum_{i=1}^d (\htheta_i - \bmtheta^*_i) \lambdamax{S\sups(i)} \norm{\bX\sups{'} H^{-1/2} \bx}[2] \notag \\
    &\leq \tilde{\lambda} \mul{\mathbf{1}}{\Delta} \sums \norm{\bX\sups{'} H^{-1/2} \bx}[2] \notag \\
    &= \tilde{\lambda} \mul{\mathbf{1}}{\Delta} \bx{'} H^{-1/2} \Big( \sums \bX\sups \bX\sups{'} \Big) H^{-1/2} \bx \notag \\
    &\leq \tilde{\lambda} \,\, \sqrt{d} \,\, \lambdamax{H^{-1/2} V H^{-1/2}} \,\, \norm{\Delta}[2] \,\, \norm{\bx}[2][2] \notag \\
    &\leq \frac{\tilde{\lambda} \sqrt{d}}{\kappa^*}  \,\, \norm{\Delta}[2] \,\, \norm{\bx}[2][2].
\end{align*}
The second inequality is due to Assumption \ref{assumption:lambda_tilde_def}. The last inequality uses a bound on $\lambdamax{H^{-1/2} V H^{-1/2}}$ given by the next lemma.

\begin{lemma}
    \label{lemma:lambdamax_HVH}
    With $\kappa^*$ defined in Lemma \ref{lemma:term1_bound}, $\lambdamax{H^{-1/2} V H^{-1/2}} \leq \frac{1}{\kappa^*}$.
\end{lemma}

Thus, $\norm{H^{-1/2} E H^{-1/2}} \leq \frac{\tilde{\lambda} \sqrt{d}}{\kappa^*} \norm{\Delta}[2]$. Using the bound on $\norm{\Delta}[2]$ from \eqref{eq:htheta_close_to_optimal} and the fact that $\kappa^* \geq \kappa_\alpha$, we get:
\begin{align}
    \label{eq:step2_middle_term_main_bound}
    \norm{H^{-1/2} E H^{-1/2}} &\leq 4 \frac{\tilde{\lambda}\sqrt{d}}{\kappa_\alpha^2} \sqrt{\frac{d + \log (1 / \delta)}{\lambdamin{V}}} \notag \\
    &\leq \frac{1}{2}.
\end{align}
Here, the last inequality requires $\lambdamin{V} \geq 64 \frac{\tilde{\lambda}^2 d}{\kappa_\alpha^4} (d + \log(1/\delta))$. Using Lemma \ref{lemma:high_prob_alpha_bound} and equations \eqref{eq:step2_second_term_breakdown}, \eqref{eq:step2_second_term_middle_term_breakdown}, and \eqref{eq:step2_middle_term_main_bound}, we get with probability at least $1 - \delta$:
\begin{equation}
    \label{eq:term2_main_bound}
    \abs{\mul{\bx}{H^{-1}E(H + E)^{-1} Z}} \leq \frac{32 \tilde{\lambda} \sqrt{d}}{\kappa_\alpha^3} \frac{(d + \log(1/\delta))}{\sqrt{\lambdamin{V}}} \norm{\bx}[V^{-1}]
\end{equation}

Using Lemma \ref{lemma:term1_bound} and equations \eqref{eq:step2} and \eqref{eq:term2_main_bound}, we get with probability at least $1 - 3\delta$:
\begin{equation}
    \label{eq:final_result}
    \abs{\mul{\bx}{\htheta - \bmtheta^*}} \leq \Big( \frac{2}{\kappa_\alpha} \sqrt{2 \log(1/\delta)} + \frac{32 \tilde{\lambda} \sqrt{d}}{{\kappa_\alpha}^3} \frac{(d + \log(1/\delta))}{\sqrt{\lambdamin{V}}} \Big) \norm{\bx}[V^{-1}].
\end{equation}
Theorem \ref{theorem:confidence_interval} follows from simplification of \eqref{eq:final_result}. See Appendix \ref{appendix:confidence_bound_technical_lemmas} for details.
\end{proof}


\section{Proof of Technical Lemmas from Appendix \ref{appendix:confidence_bound}}
\label{appendix:confidence_bound_technical_lemmas}


\subsection{Proof of Lemma \ref{lemma:lambdamin_F}}

Let $\bmtheta \in B_\alpha$. By Assumption \ref{assumption:kappa_def}, $F(\bmtheta, \bmtheta^*) \succeq \kappa_\alpha V$ for some $\kappa_\alpha > 0$. For any $\bu \in \bbR^d$ such that $\norm{\bu}[2] = 1$,
\begin{align*}
    \bu{'} F(\bmtheta, \bmtheta^*) \bu &= \bu{'} [F(\bmtheta, \bmtheta^*) - \kappa_\alpha V + \kappa_\alpha V] \bu \notag \\
    &= \bu{'}[F(\bmtheta, \bmtheta^*) - \kappa_\alpha V] \bu + \kappa_\alpha \bu{'}V\bu \notag \\
    &\geq \kappa_\alpha \bu{'}V\bu \notag \\
    &\geq \kappa_\alpha \lambdamin{V}.
\end{align*}
Taking infimum over all $\bu$ such that $\norm{\bu}[2] = 1$ on both sides, we get $\lambdamin{F(\bmtheta, \bmtheta^*)} \geq \kappa_\alpha \lambdamin{V}$. To show the second part, note that for all $\bu$ such that $\norm{\bu}[2] = 1$, we get:
\begin{align*}
    \bu{'} V \bu &= \frac{1}{\kappa_\alpha} \bu{'} [\kappa_\alpha V - F(\bmtheta, \bmtheta^*) + F(\bmtheta, \bmtheta^*)] \bu \notag \\
    &= \frac{1}{\kappa_\alpha} \Big[ \bu{'}(\kappa_\alpha V - F(\bmtheta, \bmtheta^*))\bu + \bu{'} F(\bmtheta, \bmtheta^*) \bu \Big] \notag \\
    &\leq \frac{1}{\kappa_\alpha} \bu{'} F(\bmtheta, \bmtheta^*) \bu \notag \\
    &\leq \frac{1}{\kappa_\alpha} \lambdamax{F(\bmtheta, \bmtheta^*)}.
\end{align*}
Taking supremum over all $\bu$ such that $\norm{\bu}[2] = 1$, we get $\lambdamax{V} \leq \frac{1}{\kappa_\alpha} \lambdamax{F(\bmtheta, \bmtheta^*)}$. Note that $\lambdamax{F(\bmtheta, \bmtheta^*)} = 1 / \lambdamin{F(\bmtheta, \bmtheta^*)^{-1}}$ and $\lambdamax{V} = 1 / \lambdamin{V^{-1}}$. Rearranging terms gives the desired result. In the proof for the third part, we use $F_{\bmtheta}$ as a shorthand for $F(\bmtheta, \bmtheta^*)$. For any $\bu$ such that $\norm{\bu}[2] = 1$, note that:
\begin{align*}
    \bu{'} F_{\bmtheta} V^{-1} F_{\bmtheta} \bu &= \bu{'} (F_{\bmtheta} - \kappa_\alpha V + \kappa_\alpha V) V^{-1} (F_{\bmtheta} - \kappa_\alpha V + \kappa_\alpha V) \bu \notag \\
    &= \bu{'} (F_{\bmtheta} - \kappa_\alpha V) V^{-1} (F_{\bmtheta} - \kappa_\alpha V) \bu + 2\kappa_\alpha \bu{'}(F_{\bmtheta} - \kappa_\alpha V) \bu + \kappa_\alpha^2 \bu{'} V \bu \notag \\
    &\geq \kappa_\alpha^2 \bu{'} V \bu \notag \\
    &\geq \kappa_\alpha^2 \lambdamin{V}
\end{align*}
Taking infimum over all $\bu$ such that $\norm{\bu}[2] = 1$ on both sides yields the desired result.


\subsection{Proof of Lemma \ref{lemma:high_prob_alpha_bound}}

We will use $Z$ to denote $G(\htheta)$. Note that $\norm{Z}[V^{-1}] = \norm{V^{-1/2}Z}[2] = \sup_{a : \norm{a}[2] = 1} \mul{a}{V^{-1/2}Z}$. Let $\mathbb{B}^d \coloneqq \{\bx \in \bbR^d : \norm{\bx}[2] = 1\}$ be a $d$-dimensional unit ball, and $\hat{\mathbb{B}}$ be a $1/2$-net of $\mathbb{B}^d$, i.e., for any $\bx \in \mathbb{B}^d$, there is a $\hat{\bx} \in \hat{\mathbb{B}}$ such that $\norm{\bx - \hat{\bx}}[2] \leq 1/2$. For any $\bx \in \mathbb{B}^d$,
\begin{align*}
    \mul{\bx}{V^{-1/2}Z} &= \mul{\hat{\bx}}{V^{-1/2}Z} + \mul{\bx - \hat{\bx}}{V^{-1/2}Z} \notag \\
    &= \mul{\hat{\bx}}{V^{-1/2}Z} + \norm{\bx - \hat{\bx}}[2]\mul{\frac{\bx - \hat{\bx}}{\norm{\bx - \hat{\bx}}[2]}}{V^{-1/2}Z} \notag \\
    &\leq \mul{\hat{\bx}}{V^{-1/2}Z} + \frac{1}{2} \sup_{z \in \mathbb{B}^d} \mul{z}{V^{-1/2}Z} \notag \\
    &\leq \max_{\hat{\bx} \in \hat{\mathbb{B}}} \mul{\hat{\bx}}{V^{-1/2}Z} + \frac{1}{2} \sup_{z \in \mathbb{B}^d} \mul{z}{V^{-1/2}Z}
\end{align*}
Taking supremum over $\bx \in \mathbb{B}^d$ on both sides, we get:
\begin{equation}
    \label{eq:x_G_dot_product_bound_sup}
    \norm{Z}[V^{-1}] = \sup_{\bx \in \mathbb{B}^d} \mul{\bx}{V^{-1/2}Z} \leq 2 \max_{\hat{\bx} \in \hat{\mathbb{B}}} \mul{\hat{\bx}}{V^{-1/2}Z}.
\end{equation}
Equation \eqref{eq:x_G_dot_product_bound_sup} holds trivially if $\bx \in \hat{\mathbb{B}}$ (hence $\hat{\bx} - \bx = 0$). Thus, for any $\beta > 0$,
\begin{align}
    \label{eq:prob_bound_norm_G}
    \rmP(\norm{Z}[V^{-1}] > \beta) &\leq \rmP(\max_{\hat{\bx} \in \hat{\mathbb{B}}} \mul{\hat{\bx}}{V^{-1/2}Z} > \beta/2) \notag \\
    &\leq \sum_{\hat{\bx} \in \hat{\mathbb{B}}} \rmP(\mul{\hat{\bx}}{V^{-1/2}Z} > \beta/2)
\end{align}
Recall that $Z = G(\htheta) = \sums \bX\sups \bmepsilon\sups$. For a given $\hat{\bx} \in \hat{\mathbb{B}}$, we can write $\mul{\hat{\bx}}{V^{-1/2}Z} = \sums E_{\hat{\bx}}\sups$, where $E_{\hat{\bx}}\sups = \mul{\hat{\bx}}{V^{-1/2}\bX\sups \bmepsilon\sups}$ are zero-mean independent random variables.
\begin{align*}
    E_{\hat{\bx}}\sups &= \mul{\hat{\bx}}{V^{-1/2}\bX\sups \bmepsilon\sups} \notag \\
    &= \sum_{i = 1}^K \mul{\hat{\bx}}{V^{-1/2}\bx_i\sups} (y_i\sups - \mu_i\sups(\bmtheta^*)) \notag \\
    &= \sum_{i = 1}^K \mul{\hat{\bx}}{V^{-1/2}\bx_i\sups} y_i\sups - \sum_{i = 1}^K \mul{\hat{\bx}}{V^{-1/2}\bx_i\sups} \mu_i\sups(\bmtheta^*).
\end{align*}
$E_{\hat{\bx}}\sups$ lies in an interval of size $\ell_{\hat{\bx}}\sups$ given by:
\begin{align*}
    \ell_{\hat{\bx}}\sups = \max_{i \in [K]} \mul{\hat{\bx}}{V^{-1/2}\bx_i\sups} - \min_{i \in [K]} \mul{\hat{\bx}}{V^{-1/2}\bx_i\sups} \leq 2 \max_{i \in K} \abs{\mul{\hat{\bx}}{V^{-1/2} \bx_i\sups}}
\end{align*}
Thus,
\begin{align*}
    {\ell_{\hat{\bx}}\sups}^2 &= 4 \max_{i \in K} \abs{\mul{\hat{\bx}}{V^{-1/2} \bx_i\sups}}^2 = 4 \max_{i \in [K]} \hat{\bx}{'} V^{-1/2} \bx_i\sups \bx_i\sups{'} V^{-1/2} \hat{\bx} \leq 4 \sum_{i = 1}^K \hat{\bx}{'} V^{-1/2} \bx_i\sups \bx_i\sups{'} V^{-1/2} \hat{\bx}.
\end{align*}
Using Hoeffding's inequality and the bound on $\ell_{\hat{\bx}}\sups$, we get:
\begin{align}
    \label{eq:prob_bound_norm_G_contd}
    \rmP(\mul{\hat{\bx}}{V^{-1/2}Z} > \beta/2) &\leq \exp \Big( -\frac{\beta^2}{8 \sums \sum_{i = 1}^K \hat{\bx}{'} V^{-1/2} \bx_i\sups \bx_i\sups{'} V^{-1/2} \hat{\bx}} \Big) \notag \\
    &= \exp \Big( -\frac{\beta^2}{8 \hat{\bx}{'} V^{-1/2} V V^{-1/2} \hat{\bx}} \Big) \notag \\
    &= \exp \Big( -\frac{\beta^2}{8} \Big).
\end{align}
Using \eqref{eq:prob_bound_norm_G_contd} in \eqref{eq:prob_bound_norm_G} and the fact that $\abs{\hat{\mathbb{B}}} \leq 6^d$ \cite{Pollard:1990:EmpiricalProcessesTheoryAndApplications}, we get:
\begin{align*}
    \rmP(\norm{Z}[V^{-1}] > \beta) \leq \sum_{\hat{\bx} \in \hat{\mathbb{B}}} \rmP(\mul{\hat{\bx}}{V^{-1/2}Z} > \beta/2) \leq \exp \Big( -\frac{\beta^2}{8} + d \log 6 \Big).
\end{align*}
Setting $\beta = 4 \sqrt{d + \log(1/\delta)}$ finishes the proof.


\subsection{Proof of Lemma \ref{lemma:term1_bound}}

We will use Hoeffding's inequality to bound $\rmP(\abs{\mul{\bx}{H^{-1}Z}} > \beta)$. Recall that $Z = G(\htheta) = \sums \bX\sups \bmepsilon\sups$. Thus,
\begin{align*}
    \mul{\bx}{H^{-1}Z} &= \sums \mul{\bx}{H^{-1} \bX\sups \bmepsilon\sups} =
    \sums \sum_{i = 1}^K \mul{\bx}{H^{-1} \bx_i\sups} \epsilon_i\sups = \sums E_{\bx}\sups,
\end{align*}
where, $E_{\bx}\sups = \sum_{i = 1}^K \mul{\bx}{H^{-1} \bx_i\sups} \epsilon_i\sups$ are zero mean random i.i.d. random variables. As in the proof of Lemma \ref{lemma:high_prob_alpha_bound}, we have:
\begin{align*}
    E_\bx\sups = \sum_{i = 1}^K \mul{\bx}{H^{-1} \bx_i\sups} y_i\sups - \sum_{i = 1}^K \mul{\bx}{H^{-1} \bx_i\sups} \mu_i\sups(\bmtheta^*).
\end{align*}
$E_\bx\sups$ lies in an interval of size $\ell_\bx\sups = \max_{i \in [K]} \mul{\bx}{H^{-1} \bx_i\sups} - \min_{i \in [K]} \mul{\bx}{H^{-1} \bx_i\sups}$. Note that:
\begin{align*}
    {\ell_\bx\sups}^2 &= \big(\max_{i \in [K]} \mul{\bx}{H^{-1} \bx_i\sups} - \min_{i \in [K]} \mul{\bx}{H^{-1} \bx_i\sups}\big)^2 \notag \\
    &\leq 4 \max_{i \in [K]} \mul{\bx}{H^{-1} \bx_i\sups}^2 \notag \\
    &= 4 \max_{i \in [K]} \bx{'} H^{-1} \bx_i\sups \bx_i\sups{'} H^{-1} \bx \notag \\
    &\leq 4 \sum_{i = 1}^K \bx{'} H^{-1} \bx_i\sups \bx_i\sups{'} H^{-1} \bx.
\end{align*}
Using Hoeffding's inequality, we get:
\begin{align*}
    \rmP(\abs{\mul{\bx}{H^{-1}Z}} > \beta) &\leq 2 \exp \Big( -\frac{\beta^2}{8 \sums \sum_{i = 1}^K \bx{'} H^{-1} \bx_i\sups \bx_i\sups{'} H^{-1} \bx} \Big) \notag \\
    &= 2 \exp \Big( -\frac{\beta^2}{8 \bx{'} H^{-1} V H^{-1} \bx} \Big) \notag \\
    &\leq 2 \exp \Big( -\frac{\kappa^* \beta^2}{8 \bx{'} H^{-1} H H^{-1} \bx} \Big) \notag \\
    &= 2 \exp \Big( -\frac{\kappa^* \beta^2}{8 \norm{\bx}[H^{-1}][2]} \Big),
\end{align*}
where the second inequality follows from the fact that $H \succeq \kappa^* V$. Next, we will show that $\norm{\bx}[H^{-1}][2] \leq \frac{1}{\kappa^*} \norm{\bx}[V^{-1}][2]$. As $H \succeq \kappa^*V$ and $V$ is assumed to be positive definite, we have that $\frac{1}{\kappa^*} V^{-1} \succeq H^{-1}$. Thus,
\begin{align*}
    \norm{\bx}[H^{-1}][2] &= \bx{'} H^{-1} \bx \notag \\
    &= \bx{'} (H^{-1} - \frac{1}{\kappa^*} V^{-1} + \frac{1}{\kappa^*} V^{-1}) \bx \notag \\
    &= \bx{'}(H^{-1} - \frac{1}{\kappa^*} V^{-1}) \bx + \frac{1}{\kappa^*} \bx{'} V^{-1} \bx \notag \\
    &\leq \frac{1}{\kappa^*} \norm{\bx}[V^{-1}][2].
\end{align*}
Thus, we have,
\begin{align*}
    \rmP(\abs{\mul{\bx}{H^{-1}Z}} > \beta) \leq 2 \exp \Big( -\frac{\beta^2 {\kappa^*}^2}{8 \norm{\bx}[V^{-1}][2]} \Big).
\end{align*}
Setting $\beta = \frac{2}{\kappa^*} \sqrt{2\log(1/\delta)} \norm{\bx}[V^{-1}]$ yields the desired result.


\subsection{Proof of Lemma \ref{lemma:lambdamax_HVH}}

For any $\bu \in \bbR^d$ such that $\norm{\bu}[2] = 1$,
\begin{align*}
    \bu{'}H^{-1/2}VH^{-1/2} \bu &= \frac{1}{\kappa^*} \bu{'} H^{-1/2}(\kappa^* V - H + H) H^{-1/2} \bu \notag \\
    &= \frac{1}{\kappa^*} \Big( \bu{'}H^{-1/2}(\kappa^* V - H) H^{-1/2} \bu + \bu{'}\bu \Big) \notag \\
    &\leq \frac{1}{\kappa^*}.
\end{align*}
The last inequality follows as $H \succeq \kappa^*V$ and the fact that $\bu{'}\bu = 1$. Taking supremum over all $\bu$ such that $\norm{\bu}[2] = 1$ produces the desired result.


\subsection{Simplification of Equation \ref{eq:final_result}}

We will simplify the expression in \eqref{eq:final_result} assuming that $\lambdamin{V} \geq \frac{64\tilde{\lambda}^2 d}{\kappa_\alpha^4} (d + \log (1/\delta))$.
\begin{align*}
    \abs{\mul{\bx}{\htheta - \bmtheta^*}} &\leq \frac{2}{\kappa_\alpha} \Big( \sqrt{2\log(1/\delta)} + \frac{16\tilde{\lambda} \sqrt{d}}{\kappa_\alpha^2} \frac{(d + \log(1/\delta))}{\sqrt{\lambdamin{V}}} \Big) \norm{\bx}[V^{-1}] \notag \\
    &\leq \frac{2}{\kappa_\alpha} \Big( \sqrt{2\log(1/\delta)} + 2\sqrt{d + \log(1/\delta)} \Big) \norm{\bx}[V^{-1}] \notag \\
    &\leq \frac{2}{\kappa_\alpha} \Big( 2\sqrt{d + \log(1/\delta)} + 2\sqrt{d + \log(1/\delta)} \Big) \norm{\bx}[V^{-1}] \notag \\
    &\leq \frac{8}{\kappa_\alpha} \sqrt{d + \log(1/\delta)} \norm{\bx}[V^{-1}].
\end{align*}


\section{Sample Complexity: Upper Bound}
\label{appendix:sample_complexity_upper_bound}

In this section, we present the details that were omitted from the proof sketch in Section \ref{section:sample_complexity_upper_bound}. We only need to show that $\hDelta_{1i}^{(\tau)} \geq \Delta_{\min}/2$, where recall that $\hDelta_{1i}^{(\tau)} = \mul{\htheta^{(\tau)}}{\ba_1 - \ba_i}$ and $\Delta_{\min} = \min_{i = 2, 3, \dots, N} \mul{\bmtheta^*}{\ba_1 - \ba_i}$. Note that for any $i = 2, 3, \dots, N$, $$\hDelta_{1i}^{(\tau)} = \mul{\htheta^{(\tau)}}{\ba_1 - \ba_i} = \mul{\htheta^{(\tau)} - \bmtheta^*}{\ba_1 - \ba_i} + \mul{\bmtheta^*}{\ba_1 - \ba_i}.$$ Using Corollary \ref{corollary:confidence_bound}, $$\mul{\htheta^{(\tau)} - \bmtheta^*}{\ba_1 - \ba_i} \geq -\frac{8}{\kappa_\alpha} \sqrt{d + \log(3N^2 \tau^2 / \delta)} \norm{\ba_1 - \ba_i}[{V^{(\tau)}}^{-1}] \geq -\frac{8}{\kappa_\alpha} \sqrt{d + \log(3N^2 \tau^2 / \delta)} \sqrt{\frac{\rho(\tilde{\Lambda})}{\tau K}},$$ where $\rho(\tilde{\Lambda})$ was defined in Section \ref{section:sample_complexity_upper_bound}. Thus, we get, 
\begin{align*}
    \hDelta_{1i}^{(\tau)} = \mul{\htheta^{(\tau)}}{\ba_1 - \ba_i} &\geq \mul{\bmtheta^*}{\ba_1 - \ba_i} - \frac{8}{\kappa_\alpha} \sqrt{d + \log(3N^2 \tau^2 / \delta)} \sqrt{\frac{\rho(\tilde{\Lambda})}{\tau K}} \\
    &\geq \Delta_{\min} - \frac{8}{\kappa_\alpha} \sqrt{d + \log(3N^2 \tau^2 / \delta)} \sqrt{\frac{\rho(\tilde{\Lambda})}{\tau K}}.
\end{align*}
If the algorithm stops when the sufficient stopping criterion in \eqref{eq:sufficient_stopping_criterion} is satisfied, then $\hDelta_{1i}^{(\tau)} \geq \frac{8}{\kappa_\alpha} \sqrt{d + \log(3N^2 \tau^2 / \delta)} \sqrt{\frac{\rho(\tilde{\Lambda})}{\tau K}}$. Thus,
\begin{align*}
    \hDelta_{1i}^{(\tau)} \geq \Delta_{\min} - \hDelta_{1i}^{(\tau)},
\end{align*}
and hence, $\hDelta_{1i}^{(\tau)} \geq \frac{\Delta_{\min}}{2}$.


\section{Sample Complexity: Lower Bound}
\label{appendix:sample_complexity_lower_bound}

In this section, we present the proof of Theorem \ref{theorem:lower_bound}.

\lowerbound*
\begin{proof}
We will overload the notation and use $S$ to denote both a set of arm vectors $\{\ba_{i_1}, \ba_{i_2}, \dots, \ba_{i_K}\} \subseteq_K \calA$ and the corresponding indices $\{i_1, i_2, \dots, i_K\}$. Moreover, $\mu_i\supS(\bmtheta)$ will denote the entry corresponding to the element $i \in S$ in $\bmmu\supS(\bmtheta)$ (defined in Lemma \ref{lemma:change_of_measure}). 

Because $\ba_1$ is the best arm under $\bmtheta^*$ but not under $\bmtheta^j$ defined in eq. \eqref{eq:theta_j_opt_problem}, Lemma \ref{lemma:change_of_measure} applies and we only need to compute the KL-divergence terms. For any $S \subseteq_K \calA$, we have $\mu_i\supS(\bmtheta^*) = \exp(\mul{\bmtheta^*}{\ba_i}) / c_*$ where $c_* = \sum_{k \in S} \exp(\mul{\bmtheta^*}{\ba_k})$. If $\ba_j \notin S$, then $\mu_i\supS(\bmtheta^j) = \mu_i\supS(\bmtheta^*)$ for all $i \in S$ because $\mul{\bmtheta^j}{\ba_i} = \mul{\bmtheta^*}{\ba_i}$ for all $i \neq j$ due to the equality constraint in \eqref{eq:theta_j_opt_problem}. Thus, if $\ba_j \notin S$, $\rmKL{\bmmu\supS(\bmtheta^*)}{\bmmu\supS(\bmtheta^j)} = 0$.

Now consider the case when $\ba_j \in S$. Recall that $\bmtheta^j = \bmtheta^* - \delta^j$ where $\delta^j = \frac{\epsilon + \Delta_{1j}}{\norm{\ba_1 - \ba_j}[F_j][2]} F_j(a_1 - a_j)$ for $j = 2, 3, \dots, d$. 
\begin{align*}
    \mul{\bmtheta^j}{\ba_1 - \ba_j} = \mul{\bmtheta^*}{\ba_1 - \ba_j} - \mul{\delta^j}{\ba_1 - \ba_j} = \Delta_{1j} - \epsilon - \Delta_{1j} = -\epsilon.
\end{align*}
Thus, $\mul{\bmtheta^j}{\ba_j} = \mul{\bmtheta^j}{\ba_1} + \epsilon = \mul{\bmtheta^*}{\ba_1} + \epsilon$. Hence,
\begin{equation*}
    \mu_i\supS(\bmtheta^j) = \begin{cases}
            \exp(\mul{\bmtheta^*}{\ba_1} + \epsilon) / c_j & \text{ if }i = j \\
            \exp(\mul{\bmtheta^*}{\ba_i}) / c_j & \text{ otherwise,}
    \end{cases}
\end{equation*}
where $c_j = \exp(\mul{\bmtheta^*}{\ba_1} + \epsilon) + \sum_{i \in S, i \neq j} \exp(\mul{\bmtheta^*}{\ba_i})$ is the normalizing constant as before. $\rmKL{\bmmu\supS(\bmtheta^*)}{\bmmu\supS(\bmtheta^j)}$ is given by:
\begin{align*}
    \rmKL{\bmmu\supS(\bmtheta^*)}{\bmmu\supS(\bmtheta^j)} &= \log \frac{c_j}{c_*} - (\Delta_{1j} + \epsilon) \mu_j\supS(\bmtheta^*) \notag \\
    &= \log \Big(1 + \mu_j\supS(\bmtheta^*)(\exp(\Delta_{1j} + \epsilon) - 1) \Big) - (\Delta_{1j} + \epsilon) \mu_j\supS(\bmtheta^*).
\end{align*}
For any $\alpha > 0$, the function $f_\alpha(x) = \log(1 + x(\exp(\alpha) - 1)) - \alpha x$ is maximized in the interval $0 \leq x \leq 1$ at $x^* = \frac{1}{\alpha} - \frac{1}{\exp(\alpha) - 1}$. $f_\alpha(x)$ is monotonically increasing in the interval $[0, x^*]$ and monotonically decreasing in the interval $[x^*, 1]$. Thus, if $f_\alpha(x)$ is optimized in the interval $[0, \bar{x}]$ for some $\bar{x} \leq x^*$, then the maximum will be attained at $x = \bar{x}$. Note that,
\begin{align*}
    \mu_j\supS(\bmtheta^*) &= \frac{\exp(\mul{\ba_j}{\bmtheta^*})}{\exp(\mul{\ba_j}{\bmtheta^*}) + \sum_{i \in S, i \neq j} \exp(\mul{\ba_i}{\bmtheta^*})} \\
    &= \frac{\exp(-\Delta_{1j})}{\exp(-\Delta_{1j}) + \sum_{i \in S, i \neq j} \exp(-\Delta_{1i})} \\
    &= \frac{1}{1 + \sum_{i \in S, i \neq j} \exp(\Delta_{1j}-\Delta_{1i})} \\
    &\leq \frac{1}{1 + (K - 1)/e} \\
    &\leq \frac{e}{K - 1},
\end{align*}
where the second last line follows from the assumption that $\Delta_{1i} \leq 1$ for all $i \in [N]$. Using $\alpha = \Delta_{1j} + \epsilon$, we also have,
\begin{align*}
    \frac{1}{\Delta_{1j} + \epsilon} - \frac{1}{\exp(\Delta_{1j} + \epsilon) - 1} &= \frac{\exp(\Delta_{1j} + \epsilon) - 1 - (\Delta_{1j} + \epsilon)}{(\Delta_{1j} + \epsilon)(\exp(\Delta_{1j} + \epsilon) - 1)} \\
    &\geq \frac{(\Delta_{1j} + \epsilon)^2}{2(\Delta_{1j} + \epsilon)(\exp(\Delta_{1j} + \epsilon) - 1)} \\
    &\geq \frac{1}{4}.
\end{align*}
Here, the last line uses the assumption that $\Delta_{1i} \leq 1$ for all $i \in [N]$. Thus, we have, $\mu_j\supS(\bmtheta^*) \leq \frac{e}{K - 1}$ which in turn is upper bounded by $\frac{1}{4}$ for large enough $K$. We only need to maximize $f_\alpha(x)$ for $\alpha = \Delta_{1j} + \epsilon$ in the interval $[0, e/(K - 1)]$ and because for large enough $K$, $\frac{e}{K - 1} \leq \frac{1}{4} \leq \frac{1}{\Delta_{1j} + \epsilon} - \frac{1}{\exp(\Delta_{1j} + \epsilon) - 1}$, the maximum value will be attained at $x = \frac{e}{K - 1}$. Thus, we have,
\begin{align*}
    \rmKL{\bmmu\supS(\bmtheta^*)}{\bmmu\supS(\bmtheta^j)} &\leq \log \Big(1 + \frac{e}{K - 1} (\exp(\Delta_{1j} + \epsilon) - 1) \Big) - \frac{e}{K - 1} (\Delta_{1j} + \epsilon) \\
    &\leq \frac{e}{K - 1} \Big[\exp(\Delta_{1j} + \epsilon) - 1 - (\Delta_{1j} + \epsilon) \Big] \\
    &\leq \frac{e}{K - 1} (\Delta_{1j} + \epsilon)^2.
\end{align*}

Using Lemma \ref{lemma:change_of_measure}, we get,
$$\log \frac{1}{2.4\delta} \leq \sumS \rmE_{\bmtheta^*} [N_S(\tau)] \,\, \rmKL{\bmmu\supS(\bmtheta^*)}{\bmmu\supS(\bmtheta^j)} \leq \frac{e}{K - 1}(\Delta_{1j} + \epsilon)^2 \sumS [N_S(\tau)] \bbI{j \in S}.$$

Summing over all arms $j \in [d] - \{1\}$ yields,
\begin{align*}
    K \sumS \rmE_{\bmtheta^*} [N_S(\tau)] \geq \sumS \rmE_{\bmtheta^*} [N_S(\tau)] \sum_{j = 2}^d \bbI{j \in S} \geq \frac{K - 1}{e} \sum_{j = 2}^d \frac{1}{(\Delta_{1j} + \epsilon)^2} \log\frac{1}{2.4 \delta}.
\end{align*}
The first inequality follows because at most $K$ distinct arms can belong to $S$. Rearranging, we get,
\begin{align*}
    \sumS \rmE_{\bmtheta^*} [N_S(\tau)] &\geq \frac{K - 1}{eK} \sum_{j = 2}^d \frac{1}{(\Delta_{1j} + \epsilon)^2} \log\frac{1}{2.4 \delta} = \frac{1 - 1/K}{e} \sum_{j = 2}^d \frac{1}{(\Delta_{1j} + \epsilon)^2} \log\frac{1}{2.4 \delta}.
\end{align*}
\end{proof}


\section{Alternative Arm-Selection Strategy}
\label{appendix:alternative_arm_selection}

Because for any $\ba_i, \ba_j \in \calA$, $\norm{\ba_i - \ba_j}[{V^{(t)}}^{-1}] \leq 2 \max_{k \in [N]} \norm{\ba_k}[{V^{(t)}}^{-1}]$, instead of using the arm-selection strategy in \eqref{eq:arm_selection}, we can use the following strategy:
\begin{equation}
    \label{eq:arm_selection_alternative}
    \{\bX\sups\}_{s \leq t} \leq \argmin_{\{\bX\sups\}_{s \leq t}} \max_{k \in [N]} \normvtinv{\ba_k}.
\end{equation}
For such a strategy, $\rho(\tilde{\Lambda}) \leq (1 + \beta) d$ \cite{SoareEtAl:2014:BestArmIdentificationInLinearBandits} where $\rho(\tilde{\Lambda})$ was defined in Section \ref{section:sample_complexity_upper_bound}. Under this strategy, Line 10 in Algorithm \ref{alg:static_XY_allocation} changes to $$\bx_k\supt = \argmin_{\ba \in \calA} \max_{\bar{\ba} \in \calA} \norm{\bar{\ba}}[(V^{(t - 1)} + \ba \ba{'})^{-1}][2],$$ and everything else remains unchanged. The sample complexity analysis follows the same steps as Section \ref{section:sample_complexity_upper_bound}. However, because $\rho(\tilde{\Lambda}) \leq (1 + \beta) d$ in this case, as opposed to $\rho(\tilde{\Lambda}) \leq 2 (1 + \beta) d$ in Section \ref{section:sample_complexity_upper_bound}, the final sample complexity bound changes by a constant multiplicative factor.

\begin{theorem}
    Using the stopping criterion from \eqref{eq:stopping_condition}, a $(1 + \beta)$-approximate arm-selection strategy that solves \eqref{eq:arm_selection_alternative} satisfies
    \begin{align*}
        \rmP(\tau \leq \frac{256(1 + \beta)}{\kappa_\alpha^2 \Delta_{\min}^2} (d + &\log(3N^2 \tau^2/\delta)) \frac{d}{K} \,\,\land \,\, \hat{\ba} = \ba^* ) \geq 1 - \delta,
    \end{align*}
    where $\hat{\ba}$ is the estimated best arm and $\tau$ is the number of time steps before the stopping criterion is satisfied.
\end{theorem}


\section{Adaptive Strategy}
\label{appendix:adaptive_strategy}

\begin{algorithm}[t!]
    \caption{\algnameadaptive{}}
    \label{alg:adaptive_XY_allocation}
    \begin{algorithmic}[1]
        \STATE \textbf{Input: }Set of arms $\calA$, confidence $\delta > 0$, tuning parameters $t{'}$ and $\alpha$ 
        \STATE \textbf{Initialize: } $j \leftarrow 1$, $n_0 \leftarrow d(d + 1) + 1$, $\rho_0 \leftarrow 1$, $\tilde{\calA}_1 \leftarrow \calA$, and $\tilde{\calG}_1 \leftarrow \{\bx - \by : \bx, \by \in \tilde{\calA}_1\}$
        \WHILE[Stopping criterion]{$\abs{\tilde{\calA}_j} \neq 1$}
            \STATE \textbf{Initialize batch:} $t \leftarrow 1$, $V^{(0)} \leftarrow \mathbf{0}_{d \times d}$
            \WHILE[Initial random exploration]{$t < t{'}$}
                \STATE Set $\bx_k\supt = \ba_i$ for $\ba_i \stackrel{\mathrm{unif}}{\sim} \calA$ for all $k \in [K]$
                \STATE $V\supt \leftarrow V^{(t - 1)} + \bX\supt \bX\supt{'}.$
                \STATE $t \leftarrow t + 1$
            \ENDWHILE

            \WHILE[Run static allocation within the batch]{$\rho_j/t \geq \alpha \rho_{j - 1} / n_{j - 1}$}
                \FOR[Greedy solution to \eqref{eq:arm_selection} but restricted to $\tilde{\calG}_j$]{$k \in [K]$}
                    \STATE Set $\bx_k\supt = \argmin\limits_{\ba \in \calA} \max\limits_{\bg \in \tilde{\calG}_j} \norm{\bg}[(V^{(t-1)} + \ba \ba{'})^{-1}][2]$
                    \STATE $V^{(t - 1)} \leftarrow V^{(t - 1)} + \bx_k\supt \bx_k\supt{'}$
                \ENDFOR
                \STATE $V\supt \leftarrow V^{(t - 1)}$,  $t \leftarrow t + 1$
                \STATE $\rho_j = \max_{\bg \in \tilde{\calG}_j} \bg{'} {V\supt}^{-1} \bg$
            \ENDWHILE

            \STATE \texttt{// Prepare for the next batch}
            \STATE $n_j \leftarrow t$ 
            \STATE Estimate $\htheta^{(n_j)}$ from data collected in this batch
            \STATE $\tilde{\calA}_{j + 1} = \tilde{\calA}_j$
            \FOR[Check for undominated arms]{$\ba_i \in \tilde{\calA}_{j + 1}$}
                \IF{$\exists \ba_k \in \tilde{\calA}_j$ such that $\frac{8}{\kappa_\alpha}  \sqrt{d + \log(3N^2 n_j^2/\delta)} \,\, \norm{\ba_k - \ba_i}[{V^{(n_j)}}^{-1}] \leq \hDelta^{(n_j)}_{ki}$}
                    \STATE $\tilde{\calA}_{j + 1} \leftarrow \tilde{\calA}_{j + 1} \backslash \{\ba_i\}$
                \ENDIF
            \ENDFOR
            \STATE $\tilde{\calG}_{j + 1} \leftarrow \{\bx - \by : \bx, \by \in \tilde{\calA}_{j + 1}\}$
            \STATE $j \leftarrow j + 1$
        \ENDWHILE
        \STATE \textbf{Return:} $\argmax_{\ba \in \calA} \mul{\htheta\supt}{\ba}$
    \end{algorithmic}
\end{algorithm}

\algname{} is a static-allocation strategy, i.e., it does not consider the observed rewards from the past while selecting an action, as required by Corollary \ref{corollary:confidence_bound}. However, this prevents it from adapting its behavior to the observed data. In particular, while a static allocation strategy tries to shrink the confidence set uniformly along all directions in $\calG = \{\bx - \by : \bx, \by \in \calA\}$, an adaptive strategy would focus only on directions that help in differentiating the best-arm estimate till now from the rest. Having such an adaptive strategy requires a variant of Corollary \ref{corollary:confidence_bound} that applies to non-independent sequences $\{\bX^{(s)}\}_{s > 0}$. 

\citet{SoareEtAl:2014:BestArmIdentificationInLinearBandits} resolve this issue by simply running a static-allocation strategy in batches. After each batch, arms that are deemed sub-optimal are dropped from consideration in the next batch. This reduces the directions along which the confidence set must be shrunk in each batch, but requires the data from previous batches to be discarded to satisfy the condition in Corollary \ref{corollary:confidence_bound}. Along similar lines, we develop an adaptive variant of Algorithm \ref{alg:static_XY_allocation} that works with the MNL feedback model. We refer to this variant as \algnameadaptive{} (see Algorithm \ref{alg:adaptive_XY_allocation}). 

The following lemma identifies the sub-optimal arms to discard at the end of $t$ steps (assuming the batch has length $t$). See Appendix \ref{appendix:proof_of_lemma_undominated_arms} for its proof. Following \citet{SoareEtAl:2014:BestArmIdentificationInLinearBandits}, we say that an arm $\ba_i$ is dominated if it is identified by Lemma \ref{lemma:dominated_arms} as a sub-optimal arm.

\begin{lemma}
    \label{lemma:dominated_arms}
    Let $\htheta\supt$ be the maximum likelihood estimate of parameter $\bmtheta^*$ obtained using a fixed sequence $\{\bX\sups\}_{s \leq t}$. If there exists an arm $\ba_j$ such that $$\frac{8}{\kappa_\alpha}  \sqrt{d + \log(3N^2t^2/\delta)} \,\, \normvtinv{\ba_j - \ba_i} \leq \hDelta\supt_{ji},$$ then $\ba_i$ is a sub-optimal arm.
\end{lemma}

Algorithm \ref{alg:adaptive_XY_allocation} runs in batches. Each batch has an associated set of undominated arms $\tilde{\calA}_j \subseteq \calA$ obtained via Lemma \ref{lemma:dominated_arms} using $\htheta^{(n_j)}$ estimated from the data collected in the previous batch (lines 21--26). Here, $j$ indexes the batch and $n_j$ is the number of steps for which the $j^{th}$ batch is executed. While selecting an arm to pull in the $j^{th}$ batch, the arm selection strategy only considers gaps $\tilde{\calG}_j = \{\bx - \by: \forall \bx, \by \in \tilde{\calA}_j\}$ (lines 11--14). At the end of batch $j$, the data collected from that batch is used to estimate the set of undominated arms $\tilde{\calA}_{j + 1}$ for batch $j + 1$. The algorithm terminates when the set of undominated arms is a singleton set. 

We use the same strategy as \citet{SoareEtAl:2014:BestArmIdentificationInLinearBandits} to decide the length $n_j$ of each batch. That is, $$n_j = \min \{n \in \mathbb{N} \,\, : \,\, \rho_j(n) / n \geq \alpha \rho_{j - 1}(n_{j-1}) / n_{j - 1}\}.$$ Here, $\rho_j(n) = \max_{\bg \in \tilde{\calG}_j} \bg{'} {V^{(n)}}^{-1} \bg$. Although we do not make this explicit in the notation, while computing $\rho_j(n)$, $V^{(n)}$ is calculated from the data from batch $j$ only. The parameter $\alpha$ is a tuning parameter specified by the user.

As argued before, a static allocation strategy tries to shrink the confidence interval uniformly across all directions in $\calG = \{\bx - \by : \bx, \by \in \calA\}$. Ideally, one would like to choose actions that focus on shrinking the confidence interval only along those directions that involve the optimal arm $\ba^*$, i.e., along directions in $\calG_* = \{\ba^* - \bx : \bx \in \calA\}$. Unfortunately, we cannot do this practice because the set $\calG_*$ is unknown. \algnameadaptive{} eliminates the arms (and hence the directions to consider) after each batch. Thus, $\calG = \tilde{\calG}_0 \supseteq \tilde{\calG}_1 \supseteq \tilde{\calG}_2 \supseteq \tilde{\calG}_3 \supseteq \dots$, and \algnameadaptive{} eventually enters the ideal scenario after it reaches a batch $j$ in which $\tilde{\calG}_j \subseteq \calG_*$.

Define $\calC_*\supt$ as $$\calC_*\supt = \{\bmtheta \in \bbR^d \,:\, \forall \bg \in \calG, \, \mul{\bmtheta - \bmtheta^*}{\bg} \leq \frac{8}{\kappa_\alpha}  \sqrt{d + \log(3N^2t^2/\delta)} \,\, \normvtinv{\bg}\}.$$ For each arm $\ba_i \in \calA$, we construct a set $\calC_i$ of parameter vectors $\bmtheta$ that make $\ba_i$ a best-arm, i.e., $$\calC_i = \{\bmtheta \in \bbR^d \,:\, \forall \ba \in \calA, \, \mul{\bmtheta}{\ba_i} \geq \mul{\bmtheta}{\ba}\}.$$ The set $\calC_i \cap \calC_j$ is the set of all parameters for which both $\ba_i$ and $\ba_j$ are best arms. A static allocation strategy can stop considering the direction $\ba_i - \ba_j$ if $\calC_i \cap \calC_j \cap \calC_*\supt = \Phi$.

The sample complexity of \algnameadaptive{} is governed by two quantities which we denote by $M^*$ and $N^*$ as in \citet{SoareEtAl:2014:BestArmIdentificationInLinearBandits}. $M^*$ is defined as the minimum time needed by a static allocation strategy to eliminate all directions that do not contain the best arm, i.e., eliminate all directions in $\calG - \calG_*$. $$M^* = \min \{t \in \mathbb{N} \,:\, \forall \ba_i, \ba_j \neq \ba^*, i \neq j, \, \calC_i \cap \calC_j \cap \calC_*\supt = \Phi \}$$ $N^*$ is the sample complexity of an oracle that knows $\calG_*$ and the reward gaps $\Delta_{1j} = \mul{\bmtheta^*}{\ba_1 - \ba_j}$ for all $\ba_j \in \calA$ (recall that $\ba_1$ is the best arm by assumption). Such an oracle would only shrink the confidence set along the directions in $\calG_*$ and its arm-selection strategy would focus on resolving gaps where $\Delta_{1j}$ is small (we refer the reader to \citet{SoareEtAl:2014:BestArmIdentificationInLinearBandits} for more details about the oracle). The next theorem bounds the sample complexity of \algnameadaptive{}. See Appendix \ref{appendix:proof_of_upper_bound_adaptive} for its proof.

\begin{theorem}
    \label{theorem:upper_bound_adaptive}
    If Algorithm \ref{alg:adaptive_XY_allocation} uses a $(1 + \beta)$-approximate static arm-selection strategy within each batch,
    \begin{align*}
        \rmP(\tau \leq (1 + \beta) \max \{M^*, \frac{16}{\alpha} N^*\} \log \Big( \frac{256 (d + \log(3N^2 {\tau}^2/\delta))}{K \Delta_{\min}^2 \kappa_\alpha^2} \Big) \Big / \log \Big(\frac{1}{\alpha} \Big)  \,\, \land \,\, \hat{\ba} = \ba^* ) \geq 1 - \delta,
    \end{align*}
    where $\hat{\ba}$ is the estimated best arm and $\tau$ is the number of time steps before the stopping criterion ($\abs{\tilde{\calA}_j} = 1$) is satisfied.
\end{theorem}


\subsection{Proof of Lemma \ref{lemma:dominated_arms}}
\label{appendix:proof_of_lemma_undominated_arms}

Let $\ba_j$ be an arm such that $\frac{8}{\kappa_\alpha}  \sqrt{d + \log(3N^2t^2/\delta)} \,\, \normvtinv{\ba_j - \ba_i} \leq \hDelta\supt_{ji}$. Using Corollary \ref{corollary:confidence_bound}, we get with probability at least $1 - \delta$ that, $$\abs{\mul{\htheta\supt - \bmtheta^*}{\ba_j - \ba_i}} \leq \frac{8}{\kappa_\alpha}  \sqrt{d + \log(3N^2t^2/\delta)} \,\, \normvtinv{\ba_j - \ba_i} \leq \hDelta\supt_{ji} = \mul{\htheta\supt}{\ba_j - \ba_i}.$$ Thus, with high probability, $\mul{\bmtheta^*}{\ba_j} \geq \mul{\bmtheta^*}{\ba_i}$. Hence, $\ba_i$ is a sub-optimal arm.


\subsection{Proof of Theorem \ref{theorem:upper_bound_adaptive}}
\label{appendix:proof_of_upper_bound_adaptive}

The idea is to place a high probability bound on the length of each batch $n_j$ and the number of such batches. The next lemma achieves the first goal. The proof of Lemma \ref{lemma:n_j_bound} is given in Appendix \ref{appendix:proof_of_lemma_n_j_bound}.

\begin{lemma}
    \label{lemma:n_j_bound}
    For any batch indexed by $j$, $n_j \leq (1 + \beta) \max \{M^*, \frac{16}{\alpha} N^*\}$ with probability at least $1 - \delta$. $M^*$ and $N^*$ were defined before the statement of Theorem \ref{theorem:upper_bound_adaptive}.
\end{lemma}

Recall from Section \ref{section:sample_complexity_upper_bound} that $\rho(\Lambda) = \max_{i, j \in [N]} \norm{\ba_j - \ba_i}[{V_{\Lambda}}^{-1}][2] = \max_{\bg \in \calG} \norm{\bg}[{V_{\Lambda}}^{-1}][2]$. Similarly, define $\rho^j(\Lambda)$ and $\rho^*(\Lambda)$ as
\begin{align*}
    \rho^j(\Lambda) = \max_{\bg \in \tilde{\calG}_j} \norm{\bg}[{V_{\Lambda}}^{-1}][2] & \hspace{2cm} \rho^*(\Lambda) = \max_{\bg \in \calG_*} \frac{\norm{\bg}[{V_{\Lambda}}^{-1}][2]}{\Delta_{\bg}^2},
\end{align*}
where $\Delta_{\bg} = \mul{\bmtheta^*}{\bg}$. Let $J$ be the index of a batch where the stopping condition is not satisfied, i.e., $\abs{\tilde{\calA}_{J + 1}} > 1$. Thus, there is at least one arm $\ba_i \neq \ba^*$ for which,
\begin{align}
    \label{eq:eq_1_adaptive_proof}
    \frac{8}{\kappa_\alpha}  \sqrt{d + \log(3N^2 n_J^2/\delta)} \,\, \norm{\ba_k - \ba_i}[{V^{(n_J)}}^{-1}] \geq \hDelta^{(n_J)}_{ki}, \,\,\,\, \forall \, \ba_k \in \tilde{\calA}_J,
\end{align}
where the quantities are calculated from the data collected in batch $J$. By Corollary \ref{corollary:confidence_bound},
\begin{align}
    \label{eq:eq_2_adaptive_proof}
    \hDelta^{(n_J)}_{ki} \geq \Delta_{ki} - \frac{8}{\kappa_\alpha}  \sqrt{d + \log(3N^2 n_J^2/\delta)} \,\, \norm{\ba_k - \ba_i}[{V^{(n_J)}}^{-1}] \geq \Delta_{\min} - \frac{8}{\kappa_\alpha}  \sqrt{d + \log(3N^2 n_J^2/\delta)} \,\, \norm{\ba_k - \ba_i}[{V^{(n_J)}}^{-1}].
\end{align}
The last inequality follows by taking $\ba_k = \ba^*$. Note that $\ba^*$ belongs to $\tilde{\calA}_J$ with high probability. Combining equations \eqref{eq:eq_1_adaptive_proof} and \eqref{eq:eq_2_adaptive_proof}, we get,
\begin{align*}
    \Delta_{\min} \leq \frac{16}{\kappa_\alpha}  \sqrt{d + \log(3N^2 n_J^2/\delta)} \,\, \norm{\ba_k - \ba_i}[{V^{(n_J)}}^{-1}] \leq \frac{16}{\kappa_\alpha} \sqrt{d + \log(3N^2 n_J^2/\delta)} \sqrt{\frac{\rho^J(\tilde{\Lambda}_J)}{K n_J}}.
\end{align*}
Here $\tilde{\Lambda}_J$ is the distribution over arms induced by a $(1 + \beta)$-approximate solution to eq. \eqref{eq:arm_selection} during batch $J$. The last inequality follows from the definition of $\rho^j(\Lambda)$ and from noting that $\norm{\ba_k - \ba_i}[{V^{(n_J)}}^{-1}] = \frac{1}{\sqrt{K n_J}} \norm{\ba_k - \ba_i}[{V_{\tilde{\Lambda}_J}}^{-1}]$. Thus,
\begin{align*}
    \frac{\rho^J(\tilde{\Lambda}_J)}{n_J} \geq \frac{K \Delta_{\min}^2 \kappa_\alpha^2}{256 (d + \log(3N^2 n_J^2/\delta))}. 
\end{align*}
Moreover, by the nature of the criterion used for terminating each batch (Line 10 in Algorithm \ref{alg:adaptive_XY_allocation}),
\begin{align*}
    \frac{\rho^J(\tilde{\Lambda}_J)}{n_J} \leq \alpha \frac{\rho^{J-1}(\tilde{\Lambda}_{J - 1})}{n_{J - 1}} \leq \dots \leq \alpha^J \frac{\rho^{0}(\tilde{\Lambda}_{0})}{n_{0}} \leq \alpha^J.
\end{align*}
Combining the previous two results, we get,
\begin{align*}
    \frac{K \Delta_{\min}^2 \kappa_\alpha^2}{256 (d + \log(3N^2 n_J^2/\delta))} \leq \alpha^J.
\end{align*}
Hence,
\begin{align*}
    J \leq \log \Big( \frac{256 (d + \log(3N^2 n_J^2/\delta))}{K \Delta_{\min}^2 \kappa_\alpha^2} \Big) \Big / \log \Big(\frac{1}{\alpha} \Big).
\end{align*}
Combining Lemma \ref{lemma:n_j_bound} with the bound on $J$ given above concludes the proof.


\subsection{Proof of Lemma \ref{lemma:n_j_bound}}
\label{appendix:proof_of_lemma_n_j_bound}

Let $\calC_\epsilon = \{ \bmtheta \in \bbR^d \,:\, \forall \, \bg \in \calG, \,\,\abs{\mul{\bmtheta - \bmtheta^*}{\bg}} \leq \epsilon\}$, and define $\epsilon^*$ as,
\begin{align*}
    \epsilon^* = \inf\{ \epsilon > 0 \,:\, \exists \, \ba_i, \ba_j \neq \ba^*, i \neq j,  \, \calC_i \cap \calC_j \cap \calC_\epsilon \neq \Phi\}. 
\end{align*}
By definition, $M^*$ is such that $\calC_*^{(M^*)} \subseteq \calC_{\epsilon^*}$. Thus,
\begin{align}
    \label{eq:eq_1_batchsize_bound_lemma}
    \max_{\bg \in \calG} \frac{8}{\kappa_\alpha}  \sqrt{d + \log(3N^2{M^*}^2/\delta)} \,\, \norm{\bg}[{V^{(M^*)}}^{-1}] =  \frac{8}{\kappa_\alpha}  \sqrt{d + \log(3N^2{M^*}^2/\delta)} \,\, \sqrt{\frac{\rho(\tilde{\Lambda})}{K M^*}} < \epsilon^*.
\end{align}

Now consider two cases,

\paragraph{Case 1: $\sqrt{\frac{\rho^j(\Tilde{\Lambda}_j)}{K n_j}} \geq \frac{\epsilon^* \kappa_\alpha}{8 \sqrt{d + \log(3N^2{M^*}^2/\delta)}}$:} In this case,
$$\frac{\rho(\tilde{\Lambda}_j)}{K n_j} \geq \frac{\rho^j(\Tilde{\Lambda}_j)}{K n_j} \geq \frac{\rho(\tilde{\Lambda})}{K M^*} \geq \frac{\rho(\hat{\Lambda})}{ (1 + \beta) K M^*}.$$ The first inequality follows because for any $\Lambda$, $\rho^j(\Lambda) \leq \rho(\Lambda)$ as $\tilde{\calG}_j \subseteq \calG$. The second is a consequence of eq. \eqref{eq:eq_1_batchsize_bound_lemma}. The third inequality follows because $\rho(\tilde{\Lambda})$ is a $(1 + \beta)$-approximate maximizer of $\rho(\Lambda)$ and $\rho(\Lambda)$ is maximized at $\Lambda = \hat{\Lambda}$ by definition. Because $\rho(\tilde{\Lambda}_j) \leq \rho(\hat{\Lambda})$, it must be that case that $n_j \leq (1 + \beta) M^*$.

\paragraph{Case 2: $\sqrt{\frac{\rho^j(\Tilde{\Lambda}_j)}{K n_j}} \leq \frac{\epsilon^* \kappa_\alpha}{8 \sqrt{d + \log(3N^2{M^*}^2/\delta)}}$:} Let $\hat{\Lambda}_j$ be the maximizer of $\rho^j(\Lambda)$. Then,
\begin{align}
    \label{eq:rho_j_lambda_j_upper_bound}
    \rho^j(\tilde{\Lambda}_j) \leq \rho^j(\hat{\Lambda}_j) \leq \max_{\bg \in \tilde{\calG}_j} \frac{\norm{\bg}[{V_{\hat{\Lambda}_j}}^{-1}][2]}{\Delta_{\bg}^2} \max_{\bg \in \tilde{\calG}_j} \Delta_{\bg}^2 \leq \rho^*(\hat{\Lambda}_j) \, \max_{\bg \in \tilde{\calG}_j} \Delta_{\bg}^2.
\end{align}
Algorithm \ref{alg:adaptive_XY_allocation} ensures that $\tilde{\calG}_{j - 1} \supseteq \tilde{\calG}_{j}$ for all $j$. Using Corollary \ref{corollary:confidence_bound}, for any $\bg \in \tilde{\calG}_{j}$,
\begin{align*}
    \abs{\mul{\htheta^{(n_{j - 1})} - \bmtheta^*}{\bg}} \leq \max_{\bg{'} \in \tilde{\calG}_{j - 1}} \frac{8}{\kappa_\alpha}  \sqrt{d + \log(3N^2 n_{j-1}^2/\delta)} \,\, \norm{\bg{'}}[{V^{(n_{j - 1})}}^{-1}] = \frac{8}{\kappa_\alpha}  \sqrt{d + \log(3N^2 n_{j-1}^2/\delta)} \sqrt{\frac{\rho^{j - 1}(\tilde{\Lambda}_{j - 1})}{K n_{j - 1}}}.
\end{align*}
Thus,
\begin{align*}
    \Delta_{\bg} \leq \hDelta^{(n_{j - 1})}_{\bg} + \frac{8}{\kappa_\alpha}  \sqrt{d + \log(3N^2 n_{j-1}^2/\delta)} \sqrt{\frac{\rho^{j - 1}(\tilde{\Lambda}_{j - 1})}{K n_{j - 1}}}.
\end{align*}
But $\hDelta^{(n_{j - 1})}_{\bg} \leq \frac{8}{\kappa_\alpha}  \sqrt{d + \log(3N^2 n_{j-1}^2/\delta)} \sqrt{\frac{\rho^{j - 1}(\tilde{\Lambda}_{j - 1})}{K n_{j - 1}}}$, otherwise $\bg$ would have been eliminated from $\tilde{\calG}_j$ by Lemma \ref{lemma:dominated_arms}. Thus,
\begin{align*}
    \max_{\bg \in \tilde{\calG}_j} \Delta_{\bg} \leq \frac{16}{\kappa_\alpha}  \sqrt{d + \log(3N^2 n_{j-1}^2/\delta)} \sqrt{\frac{\rho^{j - 1}(\tilde{\Lambda}_{j - 1})}{K n_{j - 1}}}
\end{align*}
Using this in \eqref{eq:rho_j_lambda_j_upper_bound}, we get,
\begin{align}
    \label{eq:eq_2_batchsize_bound_lemma}
    \rho^j(\tilde{\Lambda}_j) \leq \rho^*(\hat{\Lambda}_j) \frac{256}{\kappa_\alpha^2} (d + \log(3N^2 n_{j-1}^2/\delta)) \frac{\rho^{j - 1}(\tilde{\Lambda}_{j - 1})}{K n_{j - 1}}.
\end{align}
From this point on, we subcript the $\rho$ terms to indicate the number of steps after which they were computed. That is, $\rho_n(\tilde{\Lambda}_j)$ is computed using $\tilde{\Lambda}_j$ induced by the arms pulled in the first $n$ steps in the $j^{th}$ batch.

At time $n = n_j - 1$, the termination condition for batch $j$ is still not satisfied. Thus,
\begin{align*}
    \frac{\rho_{n}^j(\tilde{\Lambda}_j)}{n} \geq \alpha \frac{\rho_{n_{j - 1}}^{j - 1}(\tilde{\Lambda}_{j - 1})}{n_{j - 1}} \geq \alpha \frac{\rho_{n_j}^j(\tilde{\Lambda}_j)}{\rho_{n_j}^*(\hat{\Lambda}_j)} \frac{\kappa_\alpha^2 K}{256 (d + \log(3N^2 n_{j-1}^2/\delta))},
\end{align*}
where the last step follows from eq. \eqref{eq:eq_2_batchsize_bound_lemma}.
Multiplying and dividing by $\rho^*_{N^*} / N^*$ where $\rho^*_{N^*} = \rho^*(\tilde{\Lambda})$ and $\tilde{\Lambda}$ corresponds to allocation returned by the oracle in the first $N^*$ steps, we get,
\begin{align*}
    \frac{\rho_n^j(\tilde{\Lambda}_j)}{n} \geq \alpha \frac{\rho_{n_j}^j(\tilde{\Lambda}_j)}{\rho_{n_j}^*(\hat{\Lambda}_j)} \frac{\rho^*_{N^*}}{N^*} \frac{\kappa_\alpha^2 K N^*}{256 (d + \log(3N^2 n_{j-1}^2/\delta)) \rho^*_{N^*}}.
\end{align*}
One can show that $\frac{64 (d + \log(3N^2 n_{j-1}^2/\delta)) \rho^*_{N^*}}{\kappa_\alpha^2 K N^*} \leq 1$ 
\cite{SoareEtAl:2014:BestArmIdentificationInLinearBandits}. Hence,
\begin{align*}
    \frac{\rho_n^j(\tilde{\Lambda}_j)}{n} \geq \alpha \frac{\rho_{n_j}^j(\tilde{\Lambda}_j)}{\rho_{n_j}^*(\hat{\Lambda}_j)} \frac{\rho^*_{N^*}}{4N^*}
\end{align*}
Recall that $n = n_j - 1$. Substituting this value above yields,
\begin{align*}
    n_j \leq 1 + \frac{4N^*}{\alpha} \frac{\rho_{n_j - 1}^j(\tilde{\Lambda}_j)}{\rho_{n_j}^j(\tilde{\Lambda}_j)} \frac{\rho_{n_j}^*(\hat{\Lambda}_j)}{\rho^*_{N^*}}.
\end{align*}
Using Lemma 5 from \citet{SoareEtAl:2014:BestArmIdentificationInLinearBandits}, this simplifies to $n_j \leq 1 + 16N^*/\alpha$.

\end{document}